\documentclass[sn-mathphys-num]{sn-jnl}

\usepackage{graphicx}%
\usepackage{multirow}%
\usepackage{amsmath,amssymb,amsfonts}%
\usepackage{amsthm}%
\usepackage{mathrsfs}%
\usepackage[title]{appendix}%
\usepackage{xcolor}%
\usepackage{textcomp}%
\usepackage{manyfoot}%
\usepackage{booktabs}%
\usepackage{url}
\usepackage{algpseudocode}%
\usepackage{listings}%

\usepackage{svg}
\usepackage[ruled, vlined, linesnumbered]{algorithm2e}

\newtheorem{lemma}{Lemma}
\newtheorem{definition}{Definition}%

\begin{document}

\title[Article Title]{Towards Effective and Efficient Graph Alignment without Supervision}


\author[1]{\fnm{Songyang} \sur{Chen}}\email{songyangchen@bjtu.edu.cn}
\author[1]{\fnm{Youfang} \sur{Lin}}\email{ yflin@bjtu.edu.cn}
\author*[1]{\fnm{Yu} \sur{Liu}}\email{yul@bjtu.edu.cn}
\author[1]{\fnm{Shuai} \sur{Zheng}}\email{ shuaizheng@bjtu.edu.cn}
\author[2]{\fnm{Lei} \sur{Zou}}\email{zoulei@pku.edu.cn}


\affil*[1]{\orgdiv{Beijing Jiaotong University}, 
\city{Beijing}, \postcode{100044}, 
\country{P.R. China}}

\affil[2]{\orgdiv{Peking University}, 
\city{Beijing}, \postcode{100871}, 
\country{P.R. China}}


\abstract{Unsupervised graph alignment aims to find the node correspondence across different graphs without any anchor node pairs. 
Despite the recent efforts utilizing deep learning-based techniques, such as the embedding and optimal transport (OT)-based approaches, we observe their limitations in terms of model accuracy-efficiency tradeoff. By focusing on the exploitation of local and global graph information, we formalize them as the ``local representation, global alignment'' paradigm, and present a new ``global representation and alignment'' paradigm to resolve the mismatch between the two phases in the alignment process. 
We then propose \underline{Gl}obal representation and \underline{o}ptimal transport-\underline{b}ased \underline{Align}ment (\texttt{GlobAlign}), and its variant, \texttt{GlobAlign-E}, for better \underline{E}fficiency. 
Our methods are equipped with the global attention mechanism and a hierarchical cross-graph transport cost, able to capture long-range and implicit node dependencies beyond the local graph structure. 
Furthermore, \texttt{GlobAlign-E} successfully closes the time complexity gap between representative embedding and OT-based methods, reducing OT's cubic complexity to quadratic terms. 
Through extensive experiments, our methods demonstrate superior performance, with up to a 20\% accuracy improvement over the best competitor. Meanwhile, \texttt{GlobAlign-E} achieves the best efficiency, with an order of magnitude speedup against existing OT-based methods. }




\maketitle

\section{Introduction}\label{intro}
The graph alignment problem aims to predict node correspondence between two attributed graphs based on their topological structure and node features. 
It has a wide range of applications, such as matching scholar profiles across multiple academic platforms~\cite{zhang2021balancing,tang2008arnetminer}, linking the same identity across different social networks for recommendation systems~\cite{zhang2015multiple,li2019partially,slotalign}, and identifying functionally similar proteins across species in protein-protein interaction networks~\cite{dhot,liu2017novel}. 
Since the graphs to be aligned might come from different domains, the problem is particularly challenging in the \emph{unsupervised} scenario, where no observed node correspondence are available.

Recently, it  has witnessed the blossom of deep learning-based approaches~\cite{galign,walign,gtcalign,gwl,slotalign, fusegwd,peng2023robust} for unsupervised graph alignment. 
These methods predict an alignment matrix representing the probabilities of node correspondence between the source and target graphs. 
A line of research~\cite{galign,walign,gtcalign}, namely, the \emph{embedding-based} method, follows the {``embed-then-cross-compare''} approach~\cite{slotalign}. 
They first obtain node embeddings through graph neural networks (GNNs). 
Then two nodes are deemed matched if their embeddings are sufficiently close, as determined by a specified metric such as cosine similarity~\cite{walign}. 
Another category of existing solutions~\cite{gwl,slotalign,dhot, fusegwd} treats graphs as probability distributions embedded in a specific metric space. 
By exploiting the Gromov-Wasserstein distance (GWD)~\cite{memoli2011gromov}, the graph alignment problem is reformulated as an optimal transport (OT) problem by minimizing the {total cost} of transporting distributions,  
and the key challenge lies in the specification of \emph{transport cost} function. 
We refer to them as \emph{optimal transport (OT)-based} methods. 
Both types of solutions have their advantages as well as limitations. 
The embedding-based models are simple and efficient, however, it is non-trivial to formulate 
the graph alignment problem in the unsupervised setting, leading to suboptimal performance. 
In contrast, OT-based methods share a well-defined optimization objective and achieve promising accuracy, but usually incurs excessive running time due to its complexity.

\begin{figure}[!htbp]
\centering
\includegraphics[width=0.8\textwidth]{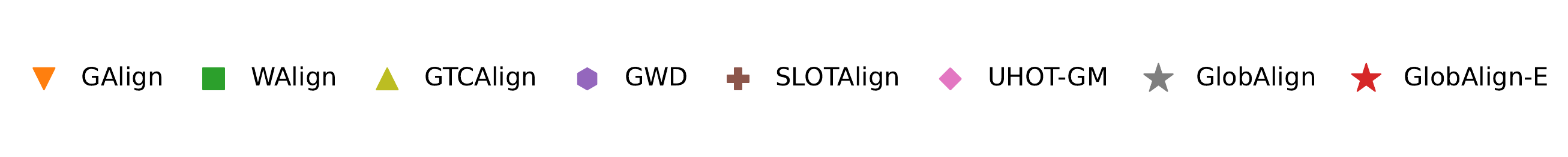} \\  
\vspace{-3mm}
\begin{tabular}{ccc} 
    \centering
    \hspace{-4mm}\includegraphics[height=33mm]{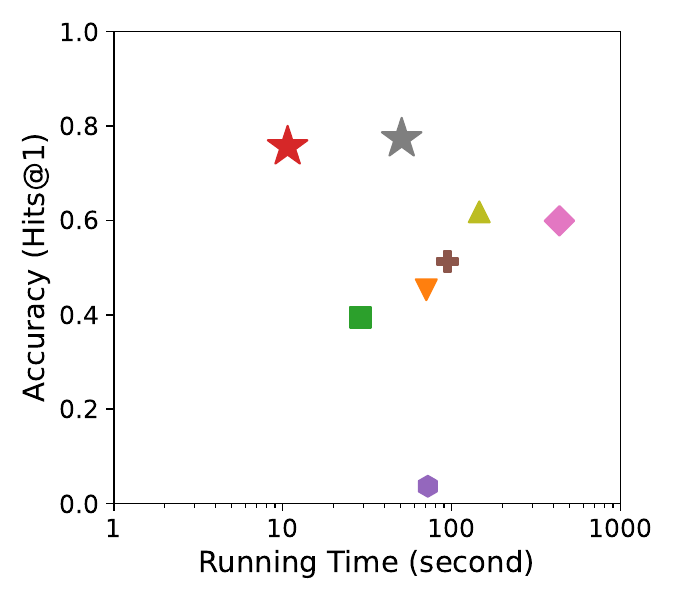} & 
    \hspace{-4mm}\includegraphics[height=33mm]{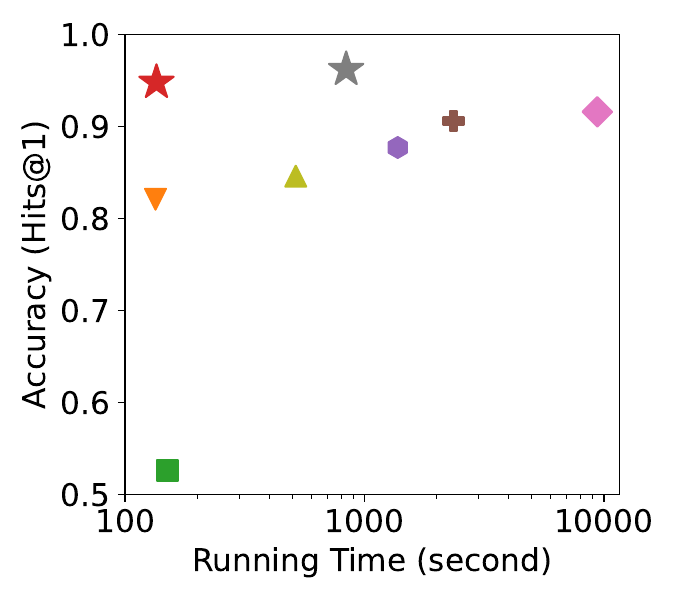} & 
    \hspace{-4mm}\includegraphics[height=33mm]{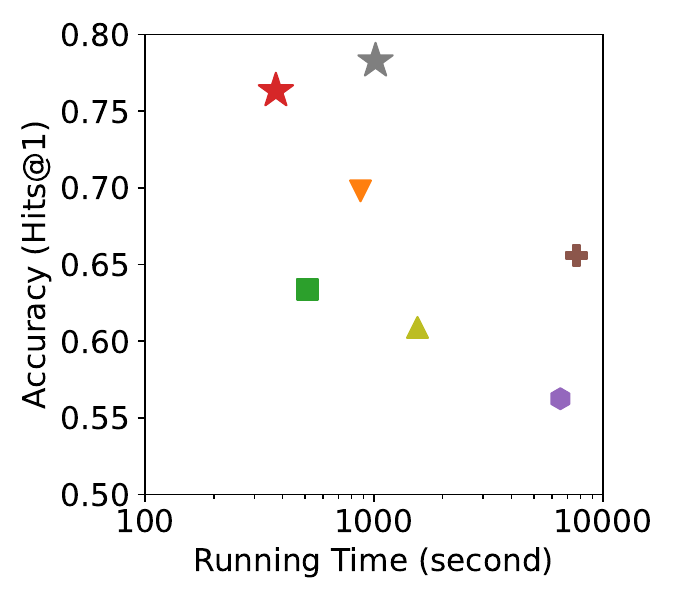}  \\
    \vspace{-4mm}
    {(a) Douban} & {(b) Allmv-Imdb} & {(c) ACM-DBLP} \\
\end{tabular}
\vspace{4mm}
\caption{Running time (s) vs. accuracy (Hits@1) on three widely-adopted datasets~\cite{slotalign,gwl,walign,gtcalign,dhot}. Existing embedding and OT-based methods show similar performance in terms of efficiency-accuracy tradeoff. Our \texttt{GlobAlign} model significantly surpasses existing solutions in accuracy, while our \texttt{GlobAlign-E} model achieves up to one order of magnitude speedup with comparable performance}
\label{fig:time}
\end{figure}



\subsection{Challenges}
In this paper, we observe two major issues overlooked by both types of studies, which hinder model performance. 

Firstly, for both categories, the representative methods~\cite{slotalign,dhot,gwl,walign,galign,gtcalign} can be summarized as a \emph{``local representation, global alignment''} paradigm. 
To be specific, the alignment procedure is decomposed into the local representation phase and the following alignment phase. 
The alignment phase is global in nature because the model predicts an alignment probability for \emph{each pair of nodes} across graphs. 
However, the local representation phase, which is crucial for node embedding and transport cost computation, relies on GNNs~\cite{gcn,gin} with local receptive fields~\cite{walign, galign, gtcalign}, local propagation~\cite{slotalign,dhot,zeng2023parrot}, or simply the adjacency information~\cite{gwl}. 
We note that the \emph{mismatch} between two phases results in model accuracy and robustness issues (Cf. Section~\ref{sec:exp}, e.g., Figure~\ref{fig:robustness-hits1}), as these models struggle to capture long-range node dependencies and those beyond the explicit graph structure. 


Secondly, despite the recent efforts (e.g.,~\cite{slotalign, gtcalign}) to enhance prediction accuracy, unfortunately, they essentially trade model efficiency for accuracy improvement as shown in our experimental study (see Figure~\ref{fig:time}). 
Take Allmv-Imdb as an example, with progressively more sophisticated transport cost design, OT-based methods (i.e., \texttt{GWD}~\cite{gwl}, \texttt{SLOTAlign}~\cite{slotalign}, and \texttt{UHOT-GM}~\cite{dhot}) indeed demonstrate enhancement in prediction accuracy. 
Nonetheless, from the viewpoint of accuracy-efficiency tradeoff, the improvement is less prominent. 
This calls for a new perspective to resolve the issue. 




\subsection{Our Contributions}
To address the aforementioned challenges, we present the following contributions.

\noindent \textbf{Formalization of Alignment Paradigm}. We first investigate existing embedding and OT-based methods by focusing on the amount of graph information exploited in the alignment process.
We formalize them as the two-phase ``local representation, global alignment'' paradigm (see Table~\ref{tab:sota comparsion}), proposing formal definitions for both phases to establish a clear understanding of model behavior. 
To the best of our knowledge, we are the first to analyze the limitations of local node interactions in existing methods for
addressing the graph alignment problem. 
Then, we theoretically analyze the mismatch between local representation and global alignment, and present a novel \emph{``global representation and alignment''} paradigm to capture long-range and implicit node dependencies.

\noindent \textbf{The \texttt{GlobAlign} Method}. To implement our proposed paradigm, we introduce an effective and efficient unsupervised graph alignment framework based on optimal transport, named \texttt{GlobAlign}. 

\emph{\underline{Gl}obal representation and \underline{o}ptimal transport-\underline{b}ased \underline{Align}ment}, 
\texttt{GlobAlign}. 
Our model leverages the self-attention mechanism~\cite{vaswani2017attention,wu2024simplifying} to derive node representations encoded with global graph information, facilitating the relationship modeling between arbitrary node pairs. 
This approach effectively addresses the challenge of capturing long-range dependencies and mitigates the issues arising from structure inconsistency~\cite{slotalign, zeng2023parrot}. 
Based on the global representation, we devise a {hierarchical cross-graph transport cost module} following the OT-based problem formation. 
It is composed of both the Wasserstein distance and Gromov-Wasserstein distance-based cost modeling which integrates global information from different perspectives. 
They interact in a complementary manner to improve model performance. 
\texttt{GlobAlign} is optimized with the well-defined OT-based objective.

\texttt{GlobAlign} \emph{ for better \underline{E}fficiency}, i.e., \texttt{GlobAlign-E}. To bridge the time complexity gap between representative embedding and OT-based methods, we propose a more scalable model variant named \texttt{GlobAlign-E} with the help of our hierarchical cost design. 
By properly exploiting global information, it retains comparable accuracy to \texttt{GlobAlign} but reduces the cubic time complexity of OT-based solutions~\cite{gwl, slotalign, dhot}, demonstrating orders of magnitude speedup. 
\texttt{GlobAlign-E} has asymptotically identical complexity compared to the embedding-based approaches~\cite{galign, gtcalign} under mild assumptions and achieves even better practical efficiency.

\noindent \textbf{Comprehensive Experimental Study}. We conduct extensive experiments, in which the contribution is two-fold. 
On the one hand, we demonstrate the limitations of existing approaches, particularly under the accuracy-efficiency tradeoff (also see Figure~\ref{fig:time}). 
On the other hand, the advantage of \texttt{GlobAlign} and \texttt{GlobAlign-E} is validated in terms of prediction accuracy and model efficiency, while the improvements are both significant. 
Besides, the robustness analysis and ablation study show the effectiveness of our model.  

To sum up, our contributions are briefed as follows:
\begin{itemize} 
\item 
\textbf{A new alignment paradigm}. 
As far as we know, we are the first to formalize graph alignment via the exploitation of (local and global) graph information and present a new ``global representation and alignment'' paradigm.

\item 
\textbf{\texttt{GlobAlign} leveraging global information}. We propose two model variants, \texttt{GlobAlign} and \texttt{GlobAlign-E}, 
using not only local graph structure but also long-range and implicit node relations. 
Furthermore, \texttt{GlobAlign-E} closes the complexity gap between embedding and OT-based approaches.


\item 
\textbf{Superior effectiveness and efficiency}.
Our models demonstrate superior performance, with up to a 20\% accuracy improvement over the best competitor. Meanwhile, \texttt{GlobAlign-E} is the fastest method in comparison, with an order of magnitude speedup against existing OT-based methods. 
\end{itemize}

\section{Preliminary and Problem} \label{sec:prelim}

\noindent \textbf{Problem Statement}.
In alignment with existing research efforts~\cite{slotalign,walign,dhot,galign,gtcalign}, we focus on the challenge of node alignment between a pair of attributed graphs in the unsupervised scenario. 
We denote an undirected and attributed graph as $\mathcal{G} = (\mathcal{V}, \mathcal{E}, \mathbf{X})$ with node set $\mathcal{V}$ of size $n$, edge set $\mathcal{E}$ represented by the adjacency matrix $\mathbf{A} \in \{0,1\}^{n \times n}$, and node features 
$\mathbf{X} \in \mathbb{R}^{n \times d}$. 
We list frequently used notations in Table~\ref{tbl:def-notation}.
 
\begin{definition} [Unsupervised Graph Alignment]
Given source graph $\mathcal{G}_s$ and target graph $\mathcal{G}_t$, we assume that $n_s = |\mathcal{V}_s|, n_t = |\mathcal{V}_t|$, and $n_s \leq n_t$ w.l.o.g. 
Without any observed node correspondences (i.e., anchors), the unsupervised graph alignment problem returns an $n_s \times n_t$-sized alignment matrix $\mathbf{T}$, where $\mathbf{T}(i,k)$ represents the probability that node $u_i\in \mathcal{G}_s$ is aligned to node $v_k\in \mathcal{G}_t$. 
\end{definition}

\noindent \textbf{Optimal Transport for Graph Alignment}. Due to its favorable properties, a line of 
studies~\cite{slotalign, dhot, gwl} employ optimal transport (OT) to formulate the graph alignment problem. 
More specifically, it supposes that two samples (i.e., the node sets), $\mathcal{V}_s=\{u_i\}_{i=1}^{n_s}$ and $\mathcal{V}_t=\{v_k\}_{k=1}^{n_t}$, are generated from the probability distributions $\boldsymbol{\mu} \in\Delta^{ n_{s} - 1}$ and $\boldsymbol{\nu} \in\Delta^{ n_{t} - 1}$ respectively, where $\Delta^{n-1}$ denotes the $(n - 1)$-Simplex. 
Optimal transport (OT) and the associated Wasserstein Distance (WD)~\cite{WD} are then employed to quantify the discrepancy between $\boldsymbol{\mu}$ and $\boldsymbol{\nu}$ given that they lie in the same space~\cite{li2023efficient}, as depicted by OT's Kantorovich formulation~\cite{kantorovich1942transfer}: 
\begin{equation}\label{eqn:WD}
    \mathrm{W}(\boldsymbol{\mu},\boldsymbol{\nu}) = \min_{\mathbf{T}\in\Pi(\boldsymbol{\mu},\boldsymbol{\nu})}\langle\mathbf{C},\mathbf{T}\rangle.
\end{equation}
Here, $\Pi(\boldsymbol{\mu}, \boldsymbol{\nu})=\left\{\mathbf{T} \in \mathbb{R}_{+}^{n_{s} \times n_{t}}: \mathbf{T} \mathbf{1}_{n_{t}}=\boldsymbol{\mu}, \mathbf{T}^{\intercal} \mathbf{1}_{n_{s}}=\boldsymbol{\nu}\right\}$ is the set of joint probability distributions. 
and it holds that $\sum_{i=1}^{n_s} \sum_{k=1}^{n_t} {\mathbf{T}(i, k)} = 1$.
For the graph alignment problem, the $(i,k)$-th entry of $\mathbf{T}$ represents the probability that node $u_i \in \mathcal{V}_s$ is aligned to $v_k \in \mathcal{V}_t$, with $\mathbf{C} \in \mathbb{R}^{n_{s}\times n_{t} }$ specifying the pairwise transport cost across two node sets.

\begin{table}[t]
\centering
\begin{small}
\caption{Table of notations.}\label{tbl:def-notation} 
\begin{tabular}{p{3cm} p{9cm}}
\toprule
\textbf{Notation} & \textbf{Description} \\
\midrule
$\mathcal{G}_s, \mathcal{G}_t$  &  The source and target graph \\ 
$\mathbf{A}_p, \mathbf{X}_p$    &  Adjacency matrix and node feature matrix ($p=s,t$) \\ 
$u_i, u_j, v_k, v_l$            &  Nodes with $u_i, u_j \in \mathcal{G}_s$ and $v_k, v_l \in \mathcal{G}_t$ \\ 
$\mathbf{D}_s, \mathbf{D}_t$    &  The relation matrices of source graph and target graph  \\  
$\mathbf{T}$    &  The alignment matrix \\ 
$\mathbf{Cost}_{gwd}, \mathbf{Cost}_{wd}$ & The cross-graph transport cost\\
$\mathbf{S}_s, \mathbf{S}_t$ & The similarity matrices of node feature\\
$\mathbf{P}_s, \mathbf{P}_t$ & The PPR matrices of source and target graph\\
$\mathbf{M}_s, \mathbf{M}_t$ & The mask matrices \\
\bottomrule
\end{tabular}
\end{small}
\end{table}

\noindent \textbf{Gromov-Wasserstein (GW) Learning}. When $\boldsymbol{\mu}$ and $\boldsymbol{\nu}$ lie in different spaces, which is common for cross-domain graph alignment, optimal transport cannot be applied directly as it is hard to define the transport cost $\mathbf{C}$. 
To overcome this limitation, the Gromov-Wasserstein Distance (GWD) is introduced~\cite{memoli2011gromov,peyre2016gromov} to measure the discrepancy between two samples by comparing their structural similarity, which is defined within each space separately by two relation matrices $\mathbf{D}_{s} \in \mathbb{R}^{n_s \times n_s}$ and $\mathbf{D}_{t} \in \mathbb{R}^{n_t \times n_t}$. 
We have
\begin{equation}\label{eqn:GWD}
    \begin{aligned}
&\mathrm{GWD}\left((\mathbf{D}_{s},\boldsymbol{\mu}),(\mathbf{D}_{t},\boldsymbol{\nu})\right) \\ &= \min_{\mathbf{T}\in\Pi(\boldsymbol{\mu}, \boldsymbol{\nu})}\sum_{i,j,k,l}\mathcal{L}\left(\mathbf{D}_{s}(u_i,u_j),\mathbf{D}_{t}(v_k, v_l)\right) \mathbf{T}(i,k) \mathbf{T}(j,l). 
\end{aligned}
\end{equation}
Let $\mathcal{L}:\mathbb{R}\times\mathbb{R}\mapsto\mathbb{R}$ be the ground cost function, e.g., the $\ell_{2}$ loss, 
we have $\mathcal{L}\left(\mathbf{D}_{s}(u_i,u_j),\mathbf{D}_{t}(v_k, v_l)\right) = {|\mathbf{D}_s(u_i,u_j) - \mathbf{D}_t(v_k, v_l)|}^2$. 
This term measures the structural similarity between $u_i,u_j \in \mathcal{G}_s$ and $v_k,v_l \in \mathcal{G}_t$. 
In other words, for likely matched node pairs $(u_i, v_k)$ and $(u_j, v_l)$ with large values of $\mathbf{T}(i,k)$ and $\mathbf{T}(j,l)$, the values of $\mathbf{D}_s(u_i,u_j)$ and $\mathbf{D}_t(v_k, v_l)$ should be close~\cite{gwl, slotalign,peyre2016gromov}.

Following~\cite{peyre2016gromov}, we reformulate Equation~\ref{eqn:GWD} as $\min_{\mathbf{T}\in\Pi(\mu,\nu)}\langle\mathcal{L}(\mathbf{D}_{s},\mathbf{D}_{t}) \odot \mathbf{T},\mathbf{T}\rangle$, where $\left( \mathcal{L}(\mathbf{D}_{s},\mathbf{D}_{t}) \odot \mathbf{T} \right)_{i,k} = \sum_{j, l} \mathcal{L}\left(\mathbf{D}_{s}(u_i, u_j), \mathbf{D}_{t}(v_k, v_l)\right) \mathbf{T}(j,l) $.
As a result, the term $\mathcal{L}(\mathbf{D}_{s},\mathbf{D}_{t}) \odot \mathbf{T} \in \mathbb{R}^{n_s \times n_t}$ can be interpreted as the transport cost (i.e., $\mathbf{C}$) in Equation~\ref{eqn:WD}~\cite{peyre2016gromov}. 
To efficiently solve the OT problem, an additional entropic regularization term~\cite{cuturi2013sinkhorn} is introduced into Equation~\ref{eqn:WD}, which transforms the problem into a strongly convex and smooth optimization problem~\cite{li2023efficient,sinkhorn1967concerning,sinkhorn}. 
For GWD, directly computation of Equation~\ref{eqn:GWD} incurs $O(n^4)$ time. 
When the ground cost $\mathcal{L}$ is decomposable~\cite{li2023efficient, peyre2016gromov}, $\mathcal{L}(\mathbf{D}_{s},\mathbf{D}_{t}) \odot \mathbf{T}$ can be calculated as $\mathbf{D}_{s}^2 \boldsymbol{\mu} \mathbf{1}^\intercal_{n_t} + \mathbf{1}_{n_s} \boldsymbol{\nu}^\intercal \mathbf{D}_{t}^{\intercal2} - 2 \mathbf{D}_{s} \mathbf{T} \mathbf{D}_{t}^\intercal$ with $O(n^3)$ time for better scalability~\cite{gwl,slotalign}.  



\begin{table}[t]
\centering

\caption{Comparison of \texttt{GlobAlign} against state of the art: alignment paradigms and time complexity. Note that all methods share a space complexity of $O(n^2)$, which is dominated by the output size.}\label{tab:sota comparsion}
\renewcommand*{\arraystretch}{1.0} 
\begin{tabular}{c c  c c c}
\toprule
{ Method } & {Alignment Paradigm}    &  { Time Complexity }    \\ 
\midrule
 \texttt{GAlign} \cite{galign} & \multirow{3}{*}{\shortstack{Local Representation, \\Global Comparison }} &   $O(  n^2d)$   \\ 
 \texttt{WAlign} \cite{walign}&  &   $O(n^2d)$    \\ 
 \texttt{GTCAlign} \cite{gtcalign} &  &   $O( n^2d)$   \\ 
\midrule
 \texttt{GWD} \cite{gwl}& \multirow{3}{*}{\shortstack{Local Propagation, \\Global Transport }}  &    $O(n^3)$    \\ 
 \texttt{SLOTAlign} \cite{slotalign}&   &     $O(n^3)$    \\ 
 \texttt{UHOT-GM} \cite{dhot}&   &       $O(n^3)$    \\ 
\midrule
 \texttt{GlobAlign} & \multirow{2}{*}{\shortstack{Global Representation, \\Global Transport }} &     $O(n^3)$    \\ 
 \texttt{GlobAlign-E}  &    &     $O(  n^2d + nm)$    \\ 
\bottomrule

\end{tabular}
            
\end{table}

\section{The Alignment Paradigm: Local vs. Global} \label{sec:analysis}
We formalize the graph alignment process of recently proposed learning-based solutions, including embedding-based~\cite{galign, walign,  gtcalign} and OT-based ones~\cite{gwl, slotalign, dhot}, into a two-phase \emph{representation-alignment} paradigm (also see Table~\ref{tab:sota comparsion}). 
We then point out that there exists a mismatch between the two steps within the current \emph{ ``local representation, global alignment"} approach, and propose a \emph{ ``global representation and transport''} solution to tackle this problem. 

\subsection{Formalizing Existing Alignment Paradigms}
Both the existing embedding (e.g.,~\cite{galign}) and OT-based (e.g.,~\cite{slotalign}) methods can be decomposed into the following two phases.

\subsubsection{Local Representation/Propagation} For existing approaches, this phase computes the node-wise representation (for embedding-based methods) or intra-graph node relation matrices (for OT-based methods) only based on the local graph information of each node. 

\begin{definition} [Local Representation] \label{def:local-rep} Let $\mathcal{N}^k(v)$ be the $k$-hop neighbors of node $v$ for $k = 0,1,\ldots, K$, with $\mathcal{N}^0(v) = v$ and $\mathcal{N}^1(v) = \mathcal{N}(v)$ (i.e., nodes adjacent to $v$). 
Let $\mathbf{R}(v) \in \mathbb{R}^r$ be an $r$-dimensional representation calculated by a (learnable) function $f$:
\begin{equation}
    \mathbf{R}(v) = f(\{ \mathcal{N}^0(v), \mathcal{N}^1(v), \ldots, \mathcal{N}^K(v) \}).
\end{equation}
We say $\mathbf{R}(v)$ is a local representation of $v$ since it contains the structural/feature information of nodes within $K$-hops from $v$, given that $K$ is smaller than the graph diameter (e.g., $K = 2$ or $3$). 
\end{definition}

Note that the above definition adapts to both categories of learning-based alignment methods. 
For embedding-based ones, e.g.,~\cite{walign}, we set $\mathbf{R}(v) = \mathbf{Z}(v) \in \mathbb{R}^d$, which can be computed using graph neural networks (GNNs). The function $f$ is specified as $\mathrm{GNN}_{\Theta}$ accordingly in which $\Theta$ stands for the learnable parameters of GNN. 
To avoid over-smoothing~\cite{over-smoothing} and over-squashing problems~\cite{over-squashing}, long-range interactions with large $K$ are prohibited. 
As for existing OT-based solutions~\cite{gwl, slotalign, dhot}, we have $\mathbf{R}(v) = \mathbf{D}(v,\cdot) \in \mathbb{R}^n$ where $\mathbf{D}(v,\cdot)$ denotes the $v$-th row of the relation matrix $\mathbf{D}$. 
Particularly, \texttt{GWD}~\cite{gwl} directly sets $\mathbf{D}(v,\cdot) = \mathbf{A}(v,\cdot)$~\cite{slotalign}, while~\cite{slotalign,dhot} compute $\mathbf{D}(v,\cdot) = \langle \mathbf{Z}'(v), \mathbf{Z}'(\cdot) \rangle$ by employing the inner product. 
Here, $\mathbf{Z}'(v)$ is calculated in a way similar to the embedding-based approaches but with a non-learnable function $g$, which is commonly implemented as \emph{local feature propagation}, for example, $\mathbf{Z}'(v) = g(\{ \mathcal{N}^0(v), \mathcal{N}^1(v), \ldots, \mathcal{N}^K(v) \}) = \sum_{k=0}^K \mathbf{A}^k(v, \cdot) \cdot \mathbf{X}$. 

\subsubsection{Global Comparison/Transport} In this phase, each node in the source graph is compared to every node in the target graph according to the node representation or the transport cost, resulting in a global comparison/transport procedure. 

\begin{definition} [Global Alignment] \label{def:global-align} Given $\mathbf{R}_s \in \mathbb{R}^{n_s \times r}$ and $\mathbf{R}_t \in \mathbb{R}^{n_t \times r}$, the global alignment process is formulated as 
\begin{equation}
    \mathbf{T}(i, k) = h(\mathbf{R}_s(u_i), \mathbf{R}_t(v_k)),
\end{equation}
where $h: \mathbb{R}^r \times \mathbb{R}^r \mapsto \mathbb{R}$ is a learnable function (or procedure). 
This phase is global, because for each $u_i \in \mathcal{G}_s$, we compute $\mathbf{T}(i, k)$ for each $v_k \in \mathcal{G}_t$ to find the one with the largest alignment probability.
\end{definition}

Existing embedding-based methods either adopt simple functions (e.g., $h(\mathbf{R}_s(u_i), \mathbf{R}_t(v_k)) = \mathbf{R}_s(u_i)^\intercal \mathbf{R}_t(v_k)$)~\cite{galign, gtcalign} or sophisticated learnable procedures such as generative adversarial network (GAN)~\cite{walign} for the global comparison of node embeddings, which can be viewed as predicting $\mathbf{S} \in \mathbb{R}^{n_s \times n_t}$, the unnormalized form of $\mathbf{T}$. 
For OT-based models, this global transport step is implemented by GW learning, given that the GWD-based transport cost $\mathbf{Cost}_{gwd}(u_i, v_k)$ is computed from $\mathbf{R}_s(u_i)$ and $\mathbf{R}_t(v_k)$.

\subsection{Mismatch between Local Representation and Global Alignment}
We claim that there is a mismatch between the local representation phase and the global alignment phase since the former encodes only local graph information which is later consumed by the latter phase. 
Subsequently, the global alignment process cannot achieve its full potential. 
We have the following lemmas.

\begin{figure}[t]
\centering
\begin{tabular}{c} 
    \centering
    \hspace{-4mm}\includegraphics[ height=20mm]{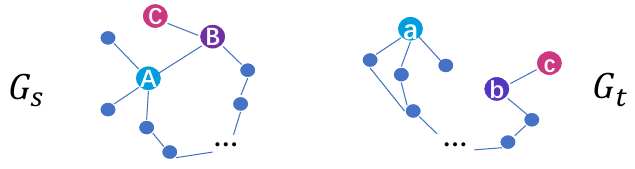}  
   
\end{tabular}
  \caption{ A toy example to show the limitation of local representation for graph alignment }
\label{fig:motivation}
\end{figure}

\begin{lemma} \label{lem:emb}
Local representation (e.g., using GNNs) is insufficient for embedding-based alignment.
\end{lemma}
\begin{proof}
Suppose that for a matched node pair $(u_i, v_k)$, we have
    \begin{equation}
    \begin{aligned}
        \mathbf{Z}_s(u_i) &= \mathrm{GNN}_\Theta(\{ \mathcal{N}_s^0(u_i), \mathcal{N}_s^1(u_i), \ldots, \mathcal{N}_s^K(u_i) \}), \\
        \mathbf{Z}_t(v_k) &= \mathrm{GNN}_\Theta(\{ \mathcal{N}_t^0(v_k), \mathcal{N}_t^1(v_k), \ldots, \mathcal{N}_t^K(v_k) \}).
    \end{aligned}
    \end{equation}
When $\mathcal{N}_s^k(u_i)$ significantly differs from $\mathcal{N}_t^k(v_k)$ (also see Figure~\ref{fig:motivation}) and $\mathrm{GNN}_\Theta$ has sufficient discriminative power, e.g., with learnable injective functions~\cite{gin}, it is possible to have $h(\mathbf{Z}_s(u_i), \mathbf{Z}_t(v_k)) < h(\mathbf{Z}_s(u_i), \mathbf{Z}_t(v_{k'}))$ for some $v_{k'}$.
\end{proof}

In general, graphs to be aligned are broadly similar but with structure inconsistency at a fine-grained level~\cite{slotalign} (e.g., matched node pairs may display different topological relationships). For instance, consider the toy example in Figure~\ref{fig:motivation}, where nodes with the same letter but different cases are aligned. 
In particular, node $u_A$ and $u_B$ are first-order neighbors, whereas their aligned counterparts, $v_a$ and $v_b$, are at a long distance. 
According to the structural similarity (e.g., rooted subtrees~\cite{gin}), it is more proper to have $\mathbf{Z}_s(u_B)^\intercal \mathbf{Z}_t(v_a) > \mathbf{Z}_s(u_B)^\intercal \mathbf{Z}_t(v_b)$, namely, $\mathbf{S}(u_B, v_a) > \mathbf{S}(u_B, v_b)$, resulting in a false alignment.



\begin{figure}[!t]
\centering

\begin{tabular}{cc} 
    \centering
    \hspace{0mm}\includegraphics[height=30mm]{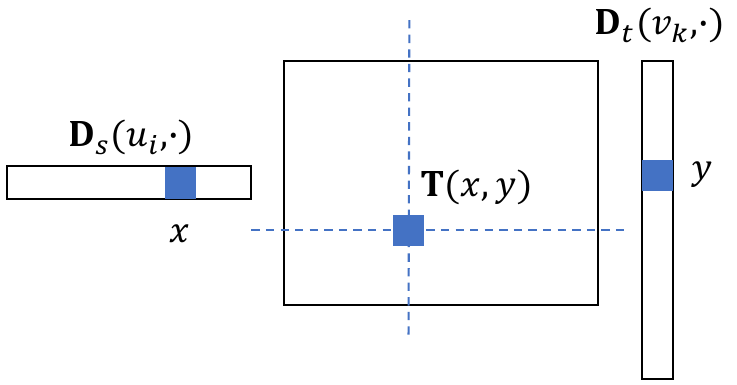} &  
\hspace{-2mm}\includegraphics[height=30mm]{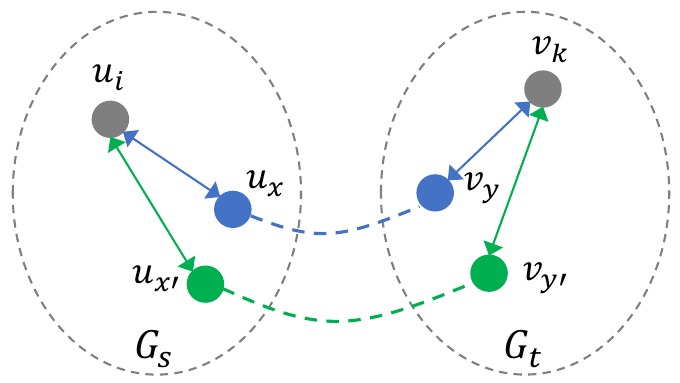}  \\
    \vspace{-3mm}
    {(a) Matrix formation} & {(b) Graph formation} \\
\end{tabular}
\vspace{3mm}
\caption{Illustration of node alignment from OT perspective}
\label{fig:example-ot}
\end{figure}

\begin{lemma} \label{lem:ot}
Computing the relation matrices with local representation (or propagation) is insufficient for OT-based alignment via GW learning.
\end{lemma}
\begin{proof}
Recall that GW learning assumes $\mathbf{D}_s(u_i, u_j) \approx \mathbf{D}_t(v_k, v_l)$ for two matched pairs $(u_i, v_k)$ and $(u_j, v_l)$. 
Therefore, the effectiveness of relation matrices is crucial to the alignment quality. 
As demonstrated in Figure~\ref{fig:example-ot}, predicting $\mathbf{T}(i, k)$ relates to the node-wise comparison on the whole set of nodes. 
When $u_i$ and $u_x$ (resp. $v_k$ and $v_y$) are beyond $2K$ steps (e.g., $v_a$ and $v_b$ in Figure~\ref{fig:motivation}), a $K$-hop aggregation causes no information intersection and thus generates less effective $\mathbf{D}_s(u_i, u_x)$ (resp. $\mathbf{D}_t(v_k, v_y)$). This contradicts the original purpose of GWD utilizing intra-graph node similarity. 
Consider a special case where nodes are associated with different one-hot features. 
In this case, local propagation (e.g.,~\cite{slotalign, dhot}) results in $\mathbf{D}_s(u_i, u_x) = \mathbf{D}_t(v_k, v_y) = 0$ for any $u_x$ and $v_y$ with more than $2K$ hops, and computes a sub-optimal transport cost. 
\end{proof}


\subsection{Global Representation and Alignment}
To bridge the gap between the representation and the alignment phase, we propose a \emph{``global representation and alignment''} paradigm under the OT-based problem formation. 
\begin{definition} [Global Representation] \label{def:global-rep}
    We say $\mathbf{R}(v) \in \mathbb{R}^r$ is node $v$'s representation with global information, if it is calculated as
    \begin{equation}
        \mathbf{R}(v) = f(\{\mathbf{X}(w), \forall w \in \mathcal{V}\}).
    \end{equation}
\end{definition}

We opt for effective mechanisms such as self-attention that yield representation with global information. 
For two nodes $u_i, u_j \in \mathcal{G}_s$ (or similarly, $v_k, v_l \in \mathcal{G}_t$), their representations can be regarded as 
\begin{equation}
    \mathbf{R}_s(u_i) = \sum_{x=1}^{n_s} \alpha_{i,x} \mathbf{X}_s(u_x), 
    \mathbf{R}_s(u_j) = \sum_{y=1}^{n_s} \alpha_{j,y} \mathbf{X}_s(u_y). 
\end{equation}
Note that $\alpha_{x, y}$ represents the attention weights, which can be properly learned to fulfill global interaction between $u_i$ and $u_j$. 
Our approach is detailed in Section~\ref{sec:model}.

\noindent \textbf{Remark}. We conjecture that it is non-trivial to apply the idea of global representation to embedding-based approaches. 
We have substituted GNNs with Transformers for representative embedding-based methods but this leads to inferior performance, possibly due to their lack of well-defined objective in the unsupervised setting. 
We leave this as future work.

\begin{figure*}[t]
    \centering
    \includegraphics[width=\linewidth, keepaspectratio]{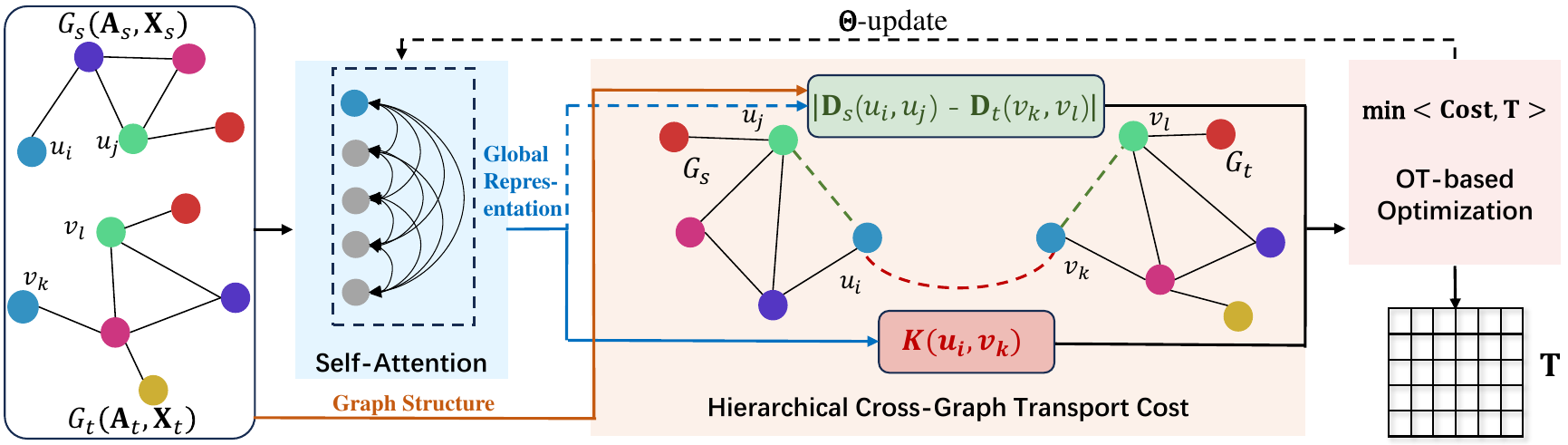} 
    \caption{The \texttt{GlobAlign} model framework}
    \label{fig:framework}
\end{figure*}

\section{Our \texttt{GlobAlign} Model} \label{sec:model}

\subsection{Model Overview}
Our proposed model, \texttt{GlobAlign}, is illustrated in Figure~\ref{fig:framework}.
Given a pair of attribute graphs, we first employ a \emph{self-attention} module (i.e., Transformers) to obtain node representation with global information. 
Next, we introduce a \emph{hierarchical cross-graph transport cost} module considering both the overall structural similarity (i.e., GWD) and node-wise similarity (i.e., WD), which is crucial for both model accuracy and efficiency. 
Finally, we minimize the total transport cost by iteratively updating both the transport cost and the alignment probabilities.

\subsection{Global Representation via Self-Attention}\label{global interactions} 
Motivated by a series of studies on Transformers~\cite{vaswani2017attention,wu2024simplifying,wu2023difformer}, we enable global interactions between nodes through the all-pair attention mechanism. 
We are particularly inspired by~\cite{wu2024simplifying} and propose the following linear attention function:
\begin{equation}
\begin{small}
\begin{aligned}
    \mathbf{Q}^{(i)} &= \mathrm{Norm}(f_Q^{(i)}(\mathbf{Z}^{(i-1)})), 
    \mathbf{K}^{(i)} = \mathrm{Norm}(f_K^{(i)}(\mathbf{Z}^{(i-1)})), 
    \mathbf{V}^{(i)} = f_V^{(i)}(\mathbf{Z}^{(i-1)}),\\
    \mathbf{D}^{(i)} &=\mathrm{diag}\left(\mathbf{1}+\frac1n \mathbf{Q}^{(i)} ( \mathbf{K}^{{(i)}\intercal} \mathbf{1})\right), 
    \mathbf{Z}^{(i)} =\mathbf{D}^{(i)-1}\left[\mathbf{V}^{(i)}+\frac1n \mathbf{Q}^{(i)} ( \mathbf{K_i} ^{{{(i)}\intercal}} \mathbf{V }^{(i)})\right],
\label{eqn:attention}
\end{aligned}
\end{small}
\end{equation}
where $f_Q^{(i)}, f_K^{(i)}$, and $f_V^{(i)}$ are linear feed-forward layers containing $O(d^2)$ parameters, respectively. We denote by $\mathrm{Norm}(\mathbf{X}) = \mathbf{X}/\|\mathbf{X}\|_F$. 
Note that $\mathbf{Z}^{(i)}$ is of dimension $n \times d$, while we set $\mathbf{Z}^{(0)} = \mathrm{MLP}(\mathbf{X})$. 

We use two layers of linear attention, which is sufficient to capture implicit dependencies between arbitrary node pairs as shown by our experimental analysis, and exhibits much better efficiency than the softmax attention~\cite{vaswani2017attention}. 
For each layer $i$, the above equation only incurs $O(nd^2)$ time. 
In order to jointly attend to information from different representation subspaces, we apply multi-head attention to compute $h$ embeddings $\mathbf{Z}_1, \ldots, \mathbf{Z}_h$, based on which the global representation $\mathbf{R}$ is derived:

\begin{equation}
\label{eqn:mutli-head}
    \mathbf{R} = \mathrm{MHA}(\mathbf{Z}_1, \ldots, \mathbf{Z}_h) = \mathrm{Concat}(\mathbf{Z}_1, \ldots, \mathbf{Z}_h)\mathbf{W}_{O}.
\end{equation}\label{eqn:mutlt-head}
Here $\mathbf{W}_{O}$ represents a $hd \times d$-sized transformation matrix.


\subsection{Hierarchical Cross-Graph Transport Cost} 
To fully leverage the global representation in the alignment phase, we present a hierarchical cross-graph transport cost design that efficiently integrates global information from different perspectives. 
Firstly, we use the Gromov-Wasserstein Distance (GWD) to model the overall structure similarity across two graphs. 
We have 
\begin{align}
    \mathbf{D}_s &= \beta_s^{(1)} \cdot \mathbf{A}_s + \beta_s^{(2)} \cdot K(\mathbf{R}_s, \mathbf{R}_s^\intercal), 
    \mathbf{D}_t = \beta_t^{(1)} \cdot \mathbf{A}_t + \beta_t^{(2)} \cdot K(\mathbf{R}_t, \mathbf{R}_t^\intercal), \nonumber\\
    &\mathbf{Cost}_{gwd}(i, k)  = \sum_{j=1}^{n_s} \sum_{l=1}^{n_t}   |\mathbf{D}_s(u_i, u_j)  - \mathbf{D}_t(v_k, v_l)|^2 \mathbf{T}(j, l), 
    \label{cost-gwd}
\end{align}
where $\boldsymbol{\beta}_s = (\beta_s^{(1)}, \beta_s^{(2)})^\intercal$ and $\boldsymbol{\beta}_t = (\beta_t^{(1)}, \beta_t^{(2)})^\intercal$ denote learnable weight parameters, while $K(\cdot, \cdot): \mathbb{R}^{n \times r} \times \mathbb{R}^{n \times r} \mapsto \mathbb{R}^{n \times n}$ maps node representation to global node relations. 
For simplicity, we implement $K(\cdot, \cdot)$ as the cosine similarity. 
Note that our node relation matrices also integrate the explicit graph structure $\mathbf{A}$ to achieve a balance between the topological inductive bias and the semantic similarity in latent space, while the latter addresses the challenge of long-range node dependencies that previous methods~\cite{slotalign,gwl,dhot,galign,walign,gtcalign} struggle to capture.

Secondly, we employ the Wasserstein Distance (WD) to directly formalize node-wise similarity with the help of global node embeddings $\mathbf{R}_s$ and $\mathbf{R}_t$. 
We exploit the function $K(\cdot, \cdot)$ and simply take the negative of node representation-based similarity since matched node pairs are with lower transport costs\footnote{As the initial features have been standardized, the similarity and cost values lie within the range of $[-1, 1]$.}:

\begin{equation}
    \mathbf{Cost}_{wd}(i,k)  = -K(\mathbf{R}_s(u_i), \mathbf{R}_t^{\intercal}[v_k]).
    \label{cost-wd}
\end{equation}

Finally, we construct the hierarchical cross-graph transport cost matrix by combining the GWD and WD costs with $\alpha$ being the weight hyperparameter:
\begin{equation}
    \mathbf{Cost}(u_i, v_k)= \alpha \cdot \mathbf{Cost}_{gwd}(u_i, v_k) + (1-\alpha) \cdot \mathbf{Cost}_{wd}(u_i, v_k).
    \label{eqn:cross-cost}
\end{equation}
Together, these two components establish a complementary relationship. 
While the GWD term exhibits more power in assessing the alignment cost, the WD term is computationally more efficient with only $O(n^2d)$ time as opposed to the $O(n^3)$ complexity of GWD.

\subsection{Improving Model Efficiency}
It is important to note the time complexity gap between embedding-based methods (e.g.,~\cite{galign, gtcalign}) and OT-based solutions (e.g.,~\cite{gwl, slotalign}), which are of $O(n^2d)$ and $O(n^3)$ complexity, respectively. 
The complexity bottleneck of OT arises from the GWD cost computation (see Equation~\ref{eqn:GWD}), more specifically, the $\mathbf{D}_{s} \mathbf{T} \mathbf{D}_{t}^\intercal$ term~\cite{gwl}, which involves the multiplication of three $n_s \times n_s$, $n_s \times n_t$, and $n_t \times n_t$-sized matrices. Notably, when the relation matrices $\mathbf{D}_s$ and  $\mathbf{D}_t$ are sparse, we can leverage sparse matrix multiplication to reduce computational complexity. 
In other words, we intend to close this complexity gap by employing a sparsified version of the hierarchical transport cost, the design of which presents non-trivial challenges to incorporate global information of the graphs. 

In this work, we propose a simple yet effective heuristic approach for the sparsification of the relation matrices (i.e., $\mathbf{D}_s$ and  $\mathbf{D}_t$).  
Since the relation matrix represents pairwise metrics between nodes within a single domain, we sparsify it by preserving only the top-$k$ (e.g., by setting $k$ as the average degree) most relevant terms for each node, determined jointly by structure similarity and semantic proximity. 
Specifically, we adopt PageRank~\cite{page1999pagerank}, a well adopted measure to model the structure similarity between nodes. In particular, we use its personalized version, i.e., the Personalized PageRank (PPR), in which the importance of node $v$ from node $u$'s perspective, denoted as $PPR(u, v)$, is defined as follows:
\begin{equation}\label{pagerank}
    PPR(u, v) = \gamma \cdot  \sum_{w \in \mathcal{N}(v)}\frac{PPR(u, w)}{\left | \mathcal{N}(w) \right |} + (1-\gamma) \cdot \mathrm{1}_{v=u}, 
\end{equation}
where $\mathrm{1}_{v=u}$ is the indicator variable which takes the value of 1 if the condition holds, and $\gamma$ is the damping factor, typically set to 0.85 by default. 
For each node $u$, we compute the PPR values and only retain the top-$k$ largest ones~\cite{PPRGo}, based on which we obtain the structure-based mask matrices $\mathbf{P}_s$ and $\mathbf{P}_t$ for source and target graphs, namely, setting their matrix entries to 1 corresponding to the top-$k$ items and 0 for others.  
To model the semantic proximity, we simply conduct feature-based similarity computation (i.e., with cosine similarity) for each graph. We sparsify the similarity matrices by only taking the row/column-wise top-$k$ largest entries, and denote the resulting semantic-based mask matrices as $\mathbf{S}_s$ and $\mathbf{S}_t$, respectively.

Subsequently, we unify the structure-based (i.e., $\mathbf{P}_s$ and $\mathbf{P}_t$) and semantic-based (i.e., $\mathbf{S}_s$ and $\mathbf{S}_t$) matrices to obtain the final mask matrices $\mathbf{M}_s$ and $\mathbf{M}_t$ in Equation~\ref{sparse-relation-matrices}, thereby achieving complementary integration of both graph structural and feature information. 
Finally, the mask operation is selectively performed only on the $K(\mathbf{R}_p, \mathbf{R}_p^\intercal)$ term because of the intrinsic sparsity of the adjacency matrix. 

\begin{equation}\label{sparse-relation-matrices}
\begin{aligned}
    \mathbf{M}_p &= \max(\mathbf{P}_p, \mathbf{S}_p), \\
    \mathbf{D}_p &= \mathbf{A}_p + \mathbf{M}_p \odot K(\mathbf{R}_p, \mathbf{R}_p^\intercal), \quad p = s,t
\end{aligned}
\end{equation}

We adopt the following cost design of $\mathbf{Cost}(u_i, v_k)$, which strikes a balance between the exploitation of global information and model efficiency: 
\begin{equation}\label{cost-sparse}
\begin{small}
    \alpha \cdot \sum_{j=1}^{n_s} \sum_{l=1}^{n_t}   |\mathbf{D}_s(u_i, u_j) - \mathbf{D}_t(v_k, v_l)|^2 \mathbf{T}(j, l)   
    +(1-\alpha)\cdot  (-K(\mathbf{R}_s(u_i), \mathbf{R}_t(v_k)^{\intercal})).
\end{small}
\end{equation}
Note that the global representation is explicitly adopted by the WD term, while the computational complexity of the GWD term is reduced thanks to the sparsification. More importantly, as both terms are unified to learn the alignment probability $\mathbf{T}$, the model accuracy is still guaranteed. 
We refer to this model version as \texttt{GlobAlign-E} (\texttt{GlobAlign} for better \texttt{E}fficiency).

\subsection{Optimization Algorithm}
Unsupervised graph alignment is solved by iteratively minimizing the inner product of the cost term $\mathbf{Cost}$ and the alignment matrix $\mathbf{T}$:
\begin{equation}
\arg\min_\mathbf{T} \mathcal{F}( \mathbf{Cost},\mathbf{T}), s.t. \mathrm{T} \mathbf{1}_{n_t}=\boldsymbol{\mu},\mathrm{T}^\intercal \mathbf{1}_{n_s} =\boldsymbol{\nu}. 
\label{obj}
\end{equation}
Equation~\ref{obj} is a nonconvex bi-quadratic problem~\cite{gwl, peyre2016gromov}. 
Let $\Theta = [\Theta_{\text{SA}}, \boldsymbol{\beta}_s, \boldsymbol{\beta}_t]$ represent the learnable parameters, where $\Theta_{\text{SA}}$ is for the self-attention module.
The basic strategy here is to optimize the $\mathbf{Cost}$ (i.e., $\Theta$-update) and the alignment matrix $\mathbf{T}$ (i.e., $\mathbf{T}$-update) in an alternating fashion. 
We adopt the proximal alternating linearized minimization strategy~\cite{bolte2014proximal}. 
For the $\Theta$-update, we have

\begin{equation}
\label{eqn:update-theta}
\boldsymbol{\Theta}^{(i+1)}  =\arg\operatorname*{min}_\Theta \left\{\nabla_{\boldsymbol{\Theta}} \mathcal{F}(\boldsymbol{\Theta}^{(i)})^\intercal \boldsymbol{\Theta} +\frac{1}{2\tau}\|\boldsymbol{\Theta}-{\boldsymbol{\Theta}}^{(i))}\|^{2}\right\}. 
\end{equation}
To update $\mathbf{T}$, we follow the approach outlined in existing works~\cite{gwl,slotalign,dhot}, where entropic regularization is applied, 
and thus we can invoke the Sinkhorn algorithm~\cite{sinkhorn} to tackle it efficiently: 

\begin{equation}
\label{eqn:updatae-T}
    \mathbf{T}^{(i+1)} = \arg\min_\mathbf{T}  \left\{ \langle \mathbf{Cost}^{(i))}, \mathbf{T}^{(i)} \rangle + \varepsilon 
  \text{KL}(\mathbf{T} \| \mathbf{T}^{(i)})  \right\}. 
\end{equation}
where the Kullback-Leibler (KL) divergence $\textrm{ KL}(\mathbf{T}\|\mathbf{T}^{(i)})$ 
serves as a regularizer, whose significance is controlled by the weight hyperparameter $\varepsilon$. 
As the cost and $\mathbf{T}$ are iteratively updated, the process terminates when the optimization objective in Equation~\ref{obj} no longer decreases.
The pseudocode is illustrated in Algorithm~\ref{alg:algorithm}.

\subsection{Complexity Analysis}\label{complexity analysis}

We provide a concise analysis of the time and space complexity of the proposed model. 
The complexity of calculating $\mathbf{Cost}_{wd}$, as defined in Equation~\ref{cost-wd}, is $O(n^2d)$. 
The cost for GWD computation following Equation~\ref{cost-gwd} results in a complexity of $O(n^3)$, Thus the proposed \texttt{GlobAlign} has a complexity of $O(n^3)$. 

For the efficient version,  \texttt{GlobAlign-E}, as formulated in  Equation~\ref{sparse-relation-matrices}, by setting $k$ to the average node degree, the sparsified relation matrices contain $O(kn + m) = O(m)$ non-zero entries, 
thus the $\mathbf{D}_{s} \mathbf{T} \mathbf{D}_{t}^\intercal$ term only has a complexity of $O(nm)$. 
More precisely, since $\mathbf{D}_{s}$ is a sparse matrix of $O(m_s)$ non-zero items, multiplying it with each column of $\mathbf{T}$ has $O(m_s)$ time complexity. 
Hence, $\mathbf{D}_{s}$ times $\mathbf{T}$ results in $O(m_s n_t)$ complexity. 
Similarly, the multiplication of a dense matrix $\mathbf{D}_{s} \mathbf{T}$ with a sparse matrix $\mathbf{D}_{t}^\intercal$ uses $O(n_s m_t)$ time. 
To this end, the GW learning process can be finished in $O(nm)$ time. Thus the proposed \texttt{GlobAlign-E} has a complexity of $O(n^2d + nm)$.

Note that real-world graphs are typically sparse, as shown in Table~\ref{table:dataset-statistics}, we have $m \leq nd$ in most cases, thus the complexity of \texttt{GlobAlign-E} is dominated by the $O(n^2d)$ term. 
It achieves identical asymptotic complexity to the state-of-the-art embedding-based methods~\cite{walign, galign, gtcalign}. 
Regarding space complexity, since the alignment matrix $\mathbf{T}$ is of size $n_s \times n_t$, the overall space complexity is bounded by $O(n^2)$.

\begin{algorithm}
  \caption{The \texttt{GlobAlign} algorithm. } \label{alg:algorithm}
  \KwIn{Attributed graphs $\mathcal{G}_s(\mathcal{V}_s, \mathcal{E}_s, \boldsymbol{X}_s), \mathcal{G}_t(\mathcal{V}_t, \mathcal{E}_t, \boldsymbol{X}_t)$} 
  \KwOut{The alignment matrix $ \mathbf{T} $}
  
  $\mathbf{Z}^{(0)}_s \leftarrow \mathrm{MLP}(\mathbf{X}_s)$,
  $\mathbf{Z}^{(0)}_t \leftarrow \mathrm{MLP}(\mathbf{X}_t$)\;
  $\boldsymbol{\mu}=\frac1{|\mathcal{V}_s|}\mathbf{1}_{|\mathcal{V}_s|}$, $\boldsymbol{\nu}=\frac1{|\mathcal{V}_t|}\mathbf{1}_{|\mathcal{V}_t|}$,
  $\mathbf{T}^{(0)}=\boldsymbol{\mu}\boldsymbol{\nu^{T}}$\;
  \For{$i = 1$ \KwTo $I$}
  {
      Compute $\mathbf{R}_s, \mathbf{R}_t$ by Equations~\ref{eqn:attention} and~\ref{eqn:mutli-head}\;
      \tcp{For \texttt{GlobAlign}} \
      Construct the hierarchical transport cost by Equations~\ref{cost-gwd},~\ref{cost-wd}, and~\ref{eqn:cross-cost}\;
     \tcp{For \texttt{GlobAlign-E}} \
     Construct the hierarchical transport cost by Equation~\ref{cost-sparse}\;
     Update $\boldsymbol{\Theta}^{(i+1)} \leftarrow \boldsymbol{\Theta}^{(i)}$ by Equation~\ref{eqn:update-theta}\;
     Update $\mathbf{T}^{(i+1)} \leftarrow \mathbf{T}^{(i)}$ by Equation~\ref{eqn:updatae-T}\;
  }

  \Return $\mathbf{T}$\;
\end{algorithm}

\section{Experiments} \label{sec:exp}
\subsection{Experimental Setup}
\subsubsection{Datasets}

\begin{table}[!htbp] 
  
  \centering
  \small
  \caption{Datasets and their statistics.} \label{table:dataset-statistics}
  \begin{tabular}{c|cccc}
    \toprule
    Dataset & $|\mathcal{V}_s|, |\mathcal{V}_t|$ & $|\mathcal{E}_s|, |\mathcal{E}_t|$ & Features & Anchors \\
    \midrule
    \multirow{2}{*}{Douban} & 1,118 & 3,022 & \multirow{2}{*}{538} & \multirow{2}{*}{1,118} \\
                                           & 3,906 & 16,328  \\
                                           \midrule
    
    \multirow{2}{*}{Allmv-Imdb} & 5,713 & 119,073 & \multirow{2}{*}{14} & \multirow{2}{*}{5,174} \\
                                           & 6,011 & 124,709  \\
                                           \midrule
    \multirow{2}{*}{ACM-DBLP} & 9,872 & 39,561 & \multirow{2}{*}{17} & \multirow{2}{*}{6,325} \\
                                           & 9,916 & 44,808  \\
                                           \midrule
    \multirow{2}{*}{CS} & 18,333 & 114,652 & \multirow{2}{*}{6,805} & \multirow{2}{*}{18,333} \\
       & 18,333 &147,410  \\
       \midrule
   \multirow{2}{*}{Physics} & 34,493 & 347,147 & \multirow{2}{*}{8,415} & \multirow{2}{*}{34,493} \\
                                       & 34,493 & 446,332  \\

    \bottomrule
\end{tabular}

\end{table}

\begin{table*}[t]
    \centering
    \small

    \resizebox{\textwidth}{!}{%
        \begin{threeparttable}
            \caption{Comparison of model performance against seven baselines on five datasets.}

            \begin{tabular}{cc|cccccccccc}
                \toprule
                Datasets &Metrics& \texttt{kNN}&\texttt{GAlign}& \texttt{WAlign} & \texttt{GTCAlign} & \texttt{GWD} & \texttt{SLOTAlign} & \texttt{UHOT-GM} & \texttt{GlobAlign} &\texttt{GlobAlign-E}\\
                \midrule
                \multirow{5}{*}{Douban} & Hits@1&27.65$\pm_{0.23}$& 45.33$\pm_{0.26}$ & 39.45$\pm_{0.37}$&{60.89}$\pm_{0.36}$ & 3.31$\pm_{0.27}$  & {51.41}$\pm_{0.12}$ & {59.91}$\pm_{0.35}$   &  \textbf{77.10$\pm_{0.18}$}& \underline{75.62$\pm_{0.16}$}\\
                &Hits@5&42.33$\pm_{0.25}$& 67.74$\pm_{0.17}$ & 62.41$\pm_{0.43}$& 76.82$\pm_{0.21}$ & 8.35$\pm_{0.35}$  & 73.48$\pm_{0.23}$ & 71.50$\pm_{0.36}$  &  \textbf{93.67$\pm_{0.21}$} &\underline{91.88$\pm_{0.27}$}\\
                &Hits@10&49.30$\pm_{0.21}$& 78.07$\pm_{0.23}$ & 71.52$\pm_{0.39}$& 82.34$\pm_{0.25}$ & 9.98$\pm_{0.31}$  & {77.71$\pm_{0.19}$} & 77.33$\pm_{0.27}$  &  \textbf{96.21$\pm_{0.25}$}&\underline{94.34$\pm_{0.22}$} \\
                 &Hits@30&63.23$\pm_{0.31}$&84.77$\pm_{0.29}$&83.21$\pm_{0.59}$&89.86$\pm_{0.23}$&12.91$\pm_{0.34}$&82.96$\pm_{0.21}$ &84.74$\pm_{0.26}$&\textbf{97.83$\pm_{0.23}$}&\underline{95.19$\pm_{0.19}$}  \\
                & MRR &35.04$\pm_{0.31}$&56.39$\pm_{0.23}$ &46.23$\pm_{0.43}$ &69.80$\pm_{0.13}$ &5.81$\pm_{0.18}$ & 61.33$\pm_{0.21}$& 67.42$\pm_{0.15}$ &\textbf{84.31$\pm_{0.17}$} &\underline{81.52$\pm_{0.13}$}\\
               
                \midrule
                
                \multirow{5}{*}{Allmv-Imdb} & Hits@1 &30.41$\pm_{0.15}$& {82.13$\pm_{0.19}$}& 52.69$\pm_{0.23}$  & 84.73$\pm_{0.16}$&87.82$\pm_{0.33}$  & 90.61$\pm_{0.11}$ & 91.63$\pm_{0.14}$ & \textbf{96.30$\pm_{0.13}$}&\underline{94.82$\pm_{0.39}$} \\
                &Hits@5&47.14$\pm_{0.20}$& 86.39$\pm_{0.23}$ & 70.91$\pm_{0.32}$& 89.89$\pm_{0.13}$ & 92.33$\pm_{0.25}$  & {92.81$\pm_{0.21}$} & 94.34$\pm_{0.18}$  &  \textbf{97.69$\pm_{0.16}$} &\underline{95.71$\pm_{0.31}$}\\
                &Hits@10&54.31$\pm_{0.41}$& 90.01$\pm_{0.22}$ & 76.52$\pm_{0.53}$ & 91.34$\pm_{0.33}$& 92.83$\pm_{0.23}$ & {93.14$\pm_{0.13}$} & 94.92$\pm_{0.16}$ &  \textbf{97.88$\pm_{0.12}$}&\underline{96.01$\pm_{0.36}$} \\
                &Hits@30&69.68$\pm_{0.33}$&93.33$\pm_{0.15}$&84.57$\pm_{0.24}$&93.71$\pm_{0.28}$&93.40$\pm_{0.18}$&94.72$\pm_{0.11}$ &96.13$\pm_{0.13}$ &\textbf{98.79$\pm_{0.09}$}&\underline{96.86$\pm_{0.14}$} \\
                & MRR&38.60$\pm_{0.37}$&84.98$\pm_{0.23}$  &61.23$\pm_{0.11}$&87.14$\pm_{0.25}$ &89.66$\pm_{0.19}$ &{91.61$\pm_{0.12}$} &92.73$\pm_{0.16}$  &\textbf{97.32$\pm_{0.13}$}&\underline{95.97$\pm_{0.23}$} \\
                \midrule
                \multirow{5}{*}{ACM-DBLP}&  Hits@1 &36.31$\pm_{0.42}$ &{70.15$\pm_{0.13}$}&63.41$\pm_{0.43}$& 60.92$\pm_{0.14}$ & 56.45$\pm_{0.23}$   & {65.96$\pm_{0.11}$} & TO  &  \textbf{78.25$\pm_{0.12}$} &\underline{75.19$\pm_{0.23}$}\\
                &Hits@5&66.81$\pm_{0.31}$ &{87.24$\pm_{0.12}$}& 83.12$\pm_{0.42}$& 75.56$\pm_{0.21}$ & 77.11$\pm_{0.13}$  & {85.38$\pm_{0.16}$} & TO  &  \textbf{93.89$\pm_{0.16}$}&\underline{91.22$\pm_{0.26}$} \\
                &Hits@10&76.29$\pm_{0.32}$&{91.44$\pm_{0.21}$} & 86.63$\pm_{0.33}$& 80.04$\pm_{0.16}$ & 82.22$\pm_{0.22}$   & {87.89$\pm_{0.15}$} & TO &  \textbf{97.02$\pm_{0.10}$}&\underline{95.21$\pm_{0.21}$} \\
                &Hits@30& 84.68$\pm_{0.42}$&94.93$\pm_{0.22}$&93.67$\pm_{0.31}$&88.12$\pm_{0.14}$&85.23$\pm_{0.23}$&90.41$\pm_{0.16}$& TO 
                &\textbf{98.38$\pm_{0.09}$}&\underline{96.45$\pm_{0.17}$}  \\
                & MRR &38.66$\pm_{0.28}$&{77.57$\pm_{0.14}$}& 70.78$\pm_{0.53}$&67.71$\pm_{0.11}$&64.81$\pm_{0.19}$  &73.78$\pm_{0.16}$ & TO &\textbf{84.52$\pm_{0.13}$} &\underline{81.96$\pm_{0.22}$}\\

                \midrule
                   \multirow{5}{*}{CS} & Hits@1 &76.14$\pm_{0.23}$& {86.71$\pm_{0.08}$}&84.33$\pm_{0.25}$   & 92.11$\pm_{0.05}$& TO & TO &TO & \textbf{99.79$\pm_{0.04}$}&\underline{98.53$\pm_{0.11}$} \\
                &Hits@5&97.45$\pm_{0.12}$&98.42$\pm_{0.06}$  &87.64$\pm_{0.17}$ &97.56$\pm_{0.04}$  & TO  &TO & TO  &  \textbf{100}&\underline{99.34$\pm_{0.13}$} \\
                &Hits@10&98.81$\pm_{0.08}$& 99.28$\pm_{0.13}$ &88.28$\pm_{0.23}$  &98.41$\pm_{0.09}$ & TO & TO & TO  &  \textbf{100}&\underline{99.71$\pm_{0.08}$} \\
                &Hits@30&99.32$\pm_{0.06}$ &99.78$\pm_{0.04}$ &89.12$\pm_{0.13}$& 99.27$\pm_{0.09}$ &TO  &TO & TO &\textbf{100}  &\textbf{100}  \\
                & MRR&85.33$\pm_{0.11}$&92.89$\pm_{0.05}$  &85.78$\pm_{0.14}$& 97.35$\pm_{0.06}$&TO &TO &TO &\textbf{99.93$\pm_{0.02}$}&\underline{99.21$\pm_{0.09}$} \\
               \midrule
                   \multirow{5}{*}{Physics} & Hits@1 &80.59$\pm_{0.16}$& {87.01$\pm_{0.23}$}&88.26$\pm_{0.29}$ &\underline{93.23$\pm_{0.11}$ } & TO  & TO & TO & TO&\textbf{99.03}$\pm_{0.15}$ \\
                &Hits@5&98.11$\pm_{0.15}$&98.19$\pm_{0.16}$  &92.83$\pm_{0.23}$ &\underline{98.47$\pm_{0.21}$}  & TO  & TO & TO  &  TO&\textbf{99.44$\pm_{0.19}$} \\
                &Hits@10&99.29$\pm_{0.09}$& 98.41$\pm_{0.12}$ &93.34$\pm_{0.18}$  &\underline{99.35$\pm_{0.18}$} & TO & TO & TO  &  TO&\textbf{100} \\
                &Hits@30& 99.56$\pm_{0.07}$&99.33$\pm_{0.11}$ &94.11$\pm_{0.18}$&\underline{99.77$\pm_{0.18}$}  & TO & TO&  TO  & TO&\textbf{100}\\
                &MRR&88.03$\pm_{0.15}$&97.56$\pm_{0.17}$  &90.67$\pm_{0.11}$&\underline{98.58$\pm_{0.07}$} &TO&TO &TO  &TO&\textbf{99.37$\pm_{0.11}$} \\
              
                \bottomrule
             
            \end{tabular}
            \begin{tablenotes}
            \footnotesize
            \item[1]The bold represents the best results, and the underlined numbers denote the second-best results.
            \item[2]TO (Time Out) stands for that the corresponding method could not finish within 3 hours.
        \end{tablenotes}
            \label{table:result}
        \end{threeparttable}
   }     
\end{table*}

Our proposed methods are evaluated on five datasets (see Table~\ref{table:dataset-statistics}), including three well-adopted real-world datasets in~\cite{walign, galign, gtcalign, gwl, slotalign}: the social network Douban Online-Offline~\cite{final}, the co-author network ACM-DBLP~\cite{zhang2018attributed}, and the movie recommendation network Allmv-Imdb~\cite{galign}. Additionally, we introduce two larger  datasets: Coauthor CS~\cite{shchur2018pitfalls} and Coauthor Physics~\cite{shchur2018pitfalls}, mainly for efficiency evaluation. 
For each graph, we create a pair of graphs by perturbing 10\% and 40\% edges in the original graph, respectively.

\subsubsection{Baselines}
We evaluate \texttt{GlobAlign} and \texttt{GlobAlign-E} against seven representative baselines, categorized as follows: 
(1) traditional method: \texttt{kNN}, (2) embedding-based methods: \texttt{GAlign}~\cite{galign}, \texttt{WAlign}~\cite{walign}, and \texttt{GTCAlign}~\cite{gtcalign}, 
(3) OT-based methods: \texttt{GWD}~\cite{gwl}, \texttt{SLOTAlign}~\cite{slotalign}, and \texttt{UHOT-GM}~\cite{dhot}. We exclude approaches such as those presented in \cite{fusegwd, wang2018cross, heimann2018regal,xu2019scalable}, as they have been consistently outperformed by the baseline methods under consideration. Additionally, we omit methods (such as~\cite{hermanns2023grasp}) that impose strict constraints (e.g., requiring the graphs to be of identical size for alignment), or those conducting exact alignment via bipartite graph matching, leading to a computational complexity of at least $O(n^3)$~\cite{hermanns2023grasp,jonker1988shortest,kuhn1955hungarian}.

\subsubsection{Metrics.}
Following~\cite{walign,slotalign,gtcalign}, we employ Hits@$k$ with $k = \{1, 5, 10, 30\}$ and Mean Reciprocal Rank (MRR) to evaluate the effectiveness of all methods. Given a node $u \in \mathcal{G}_s$, if the aligned node $v \in \mathcal{G}_t$ is among the top-$k$ nodes with the largest alignment probability, it is considered as a hit. For a dataset with $|\mathcal{S}^*|$ ground truth, Hits@k is defined as $\textrm{Hits@}k=\frac{\#\mathrm{~of~hits}}{|\mathcal{S}^*|}$.
MRR is computed by averaging the inverse of the ground truth ranking, i.e., $\textrm{MRR}=\frac{1}{|\mathcal{S}^*|}\sum_{(u, v) \in \mathcal{S}^*}\frac{1}{\mathrm{rank}(u,v)}$. 
Notably, Hits@1 and MRR directly reflect the alignment accuracy. 
We also report each method's running time to evaluate efficiency. For learning-based models, we define the running time as the total elapsed time until model convergence.
Our codes will be made available after the review process.

\subsection{Comparison of Model Performance}

\subsubsection{Model Accuracy}
The results are presented in Table~\ref{table:result}. 
We conduct 10 runs of experiments for each method. Due to the high complexity of certain methods~\cite{gwl,slotalign,dhot}, experiments were terminated if they could not be completed within three hours. 
Our proposed algorithms \texttt{GlobAlign} and \texttt{GlobAlign-E}, 
demonstrate substantial accuracy gains over existing methods across all datasets. 
For instance, on the Douban dataset, the proposed methods achieve notable improvements of 26.62\% and 24.19\% compared to \texttt{GTCAlign}, the best baseline, respectively. 
Similarly, on the DBLP dataset, the proposed approaches outperform the state-of-the-art method  by 11.55\% and 7.18\%, respectively. 
On the  datasets CS and Physics, while all methods exhibit better performance, our methods consistently achieve further enhancements. 
These results highlight the effectiveness of our ``global representation and alignment'' paradigm, especially the critical role of global representation.

For different versions of the proposed method, \texttt{GlobAlign-E} generally achieves comparable accuracy against \texttt{GlobAlign}. This also validates the effectiveness of our hierarchical transport cost design, particularly, Equation~\ref{cost-sparse}. 
For other OT-based methods including \texttt{GWD}, \texttt{SLOTAlign}, and \texttt{UHOT-GM}, 
we note that their performance demonstrates a gradual improvement along with the progressive incorporation of richer information into relation matrix design.

Last but not least, we would like to point out that both embedding and OT-based baselines generally struggle with the accuracy-efficiency tradeoff on the real-world datasets (Cf. Figure~\ref{fig:time} in Section~\ref{intro}). 
In other words, the improvement in prediction accuracy is achieved at the cost of less efficiency. 
By contrast, our methods significantly outperform them with respect to the tradeoff.

\begin{figure*}[t]
\centering
\resizebox{\textwidth}{!}{
\begin{tabular}{ccccc} 
    \centering
    \hspace{-4mm}\includegraphics[ height=30mm]{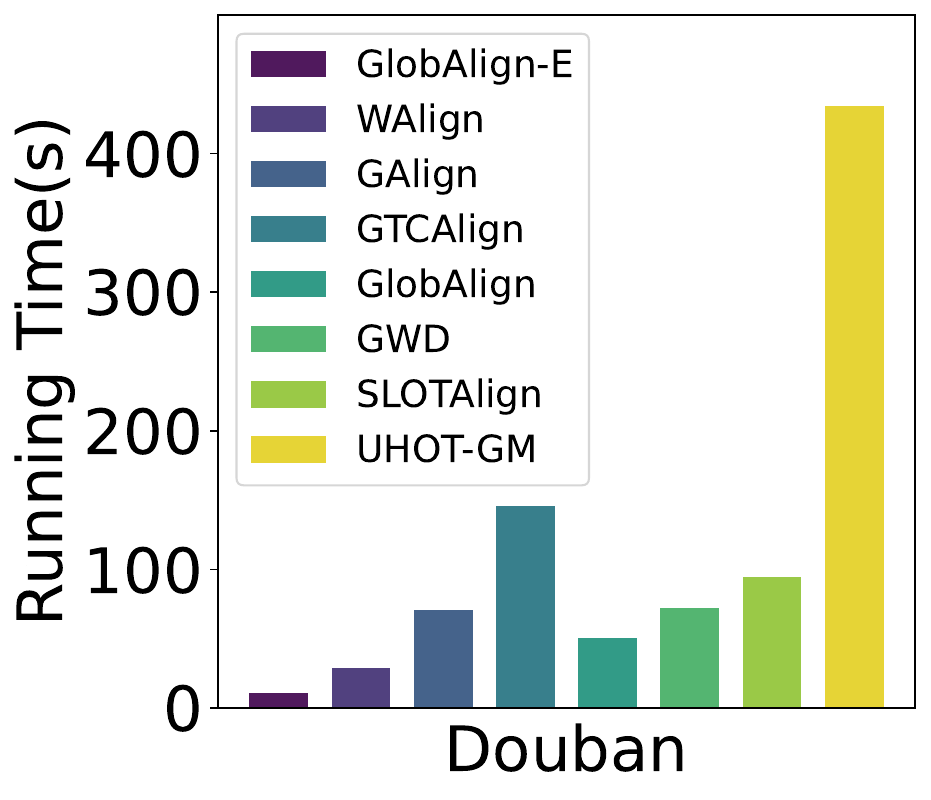} &
    \hspace{-4mm}\includegraphics[ height=30mm]{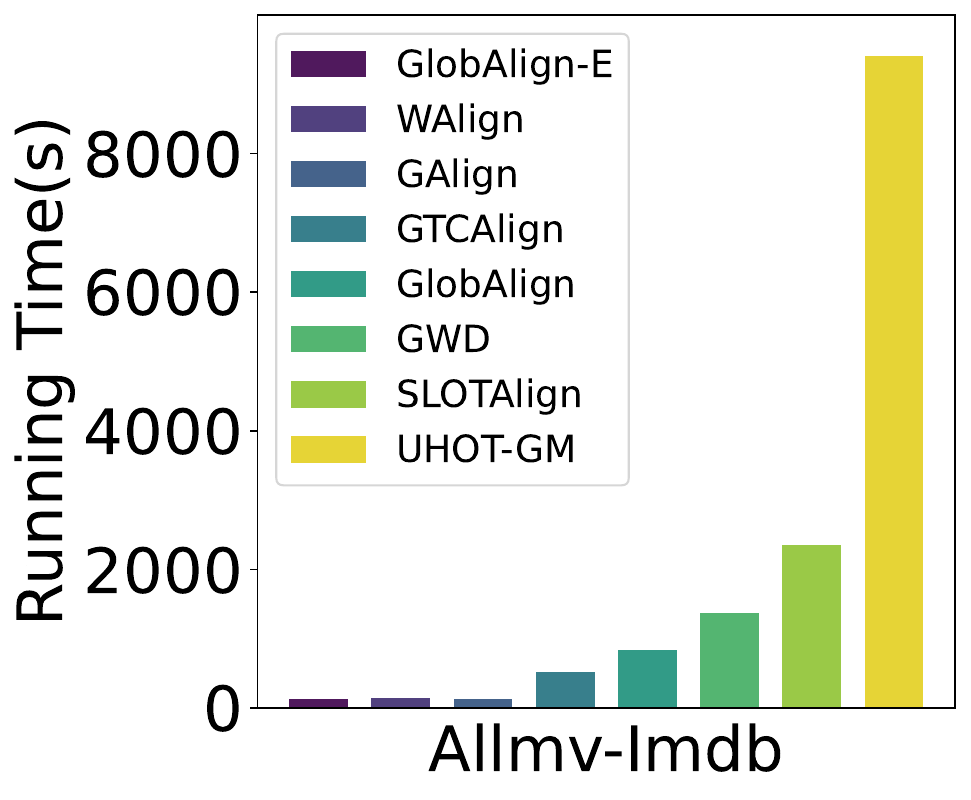} &
    \hspace{-4mm}\includegraphics[ height=30mm]{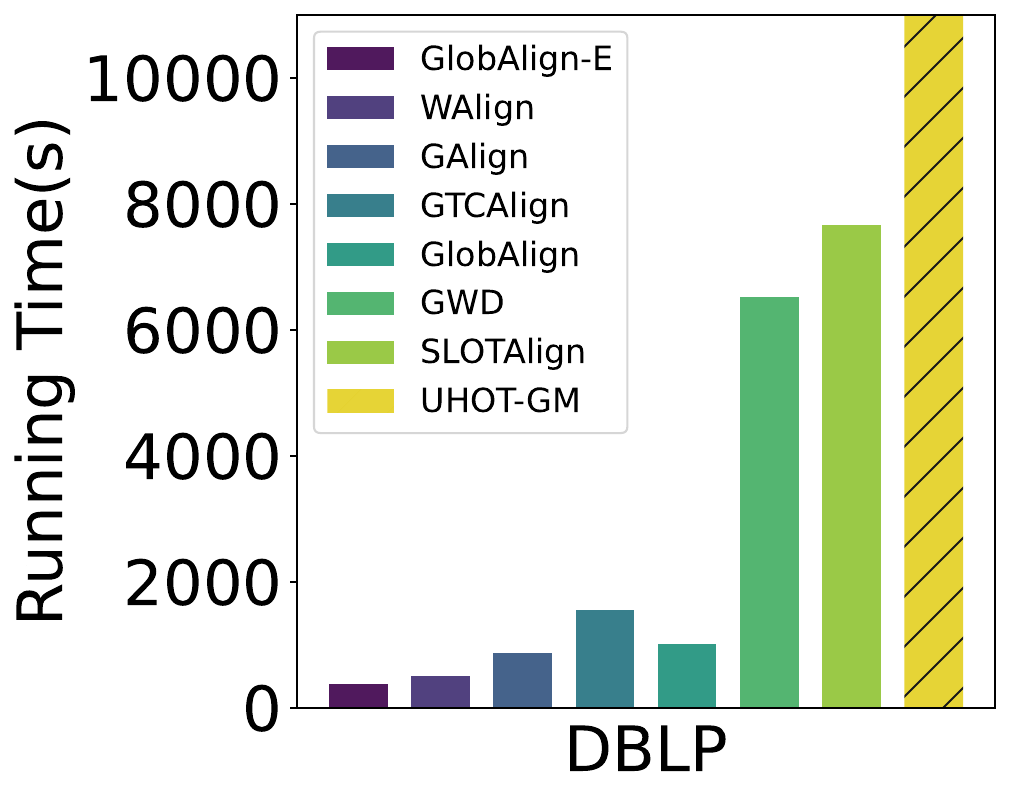} &
    \hspace{-4mm}\includegraphics[ height=30mm]{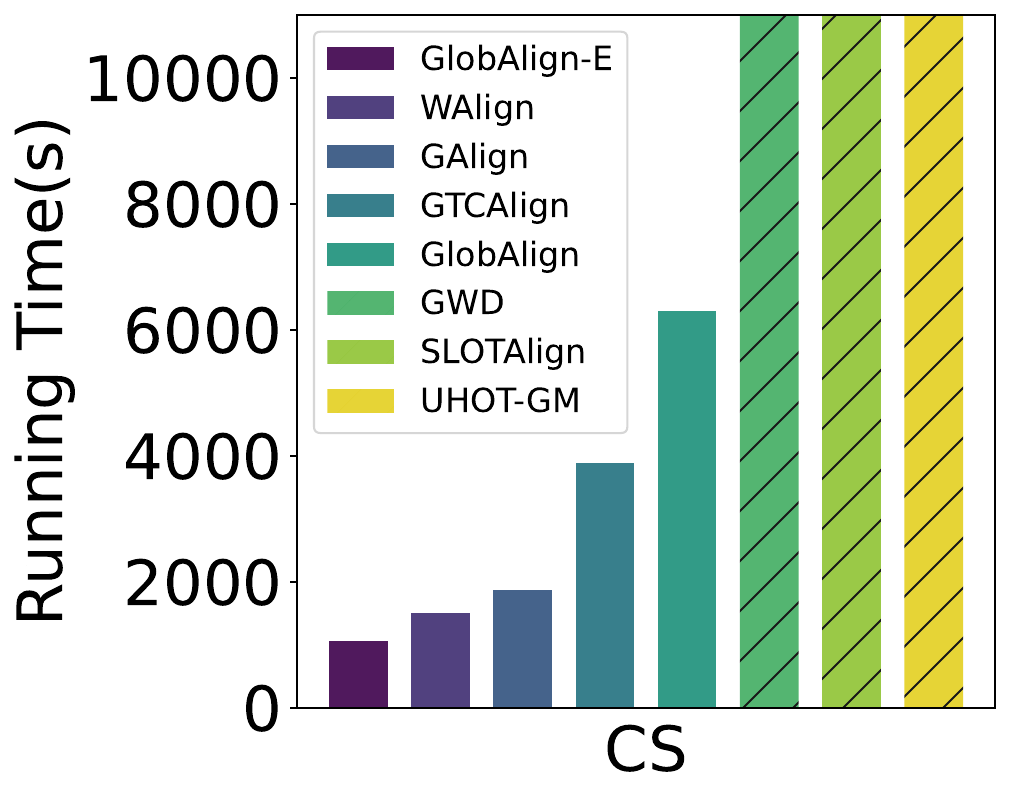} & 
    \hspace{-4mm}\includegraphics[ height=30mm]{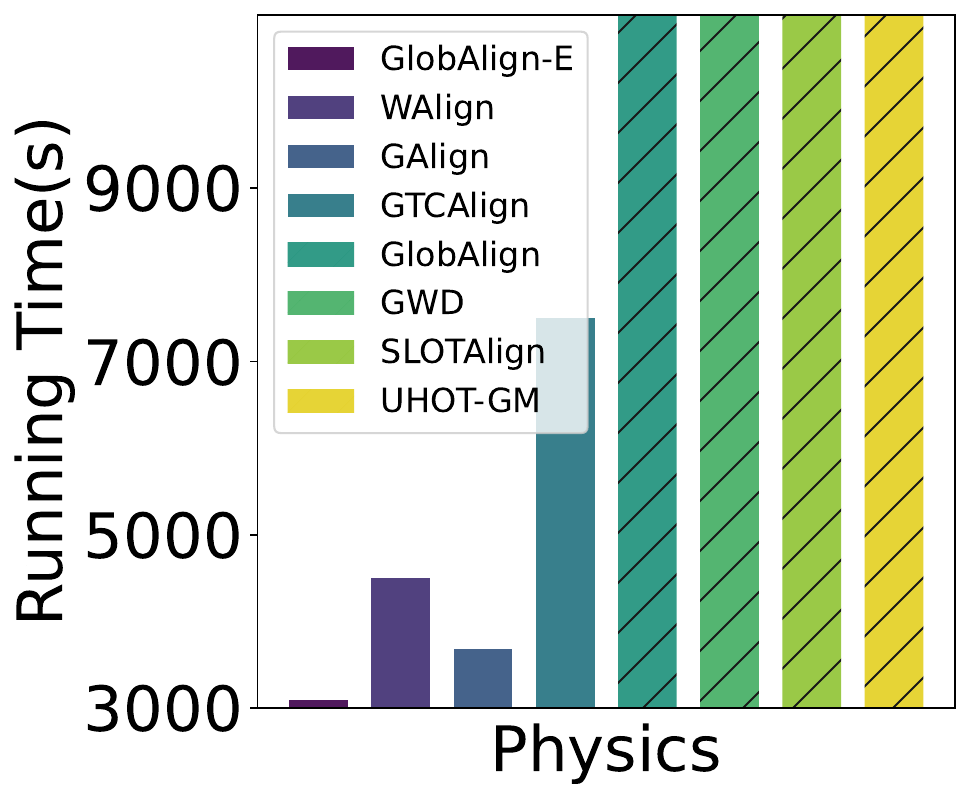} \\
\end{tabular}
}
\vspace{-3mm}
\caption{Running time comparison. Vertical bars with diagonal lines represent methods that exceed the 3-hour time limit }
\label{fig:running time}

\end{figure*}

\subsubsection{Model Efficiency}
As shown in Figure~\ref{fig:running time}, We visualize the running time of all methods on datasets of varying sizes. 
When the number of nodes reaches $10^4$, as in the DBLP dataset, \texttt{UHOT-GM} first fails to complete the experiment within the three-hour limit due to its high computational complexity. 
As the dataset size continues to grow, other OT-based methods, \texttt{GWD} and \texttt{SLOTAlign}, also fail to complete the experiment in a reasonable time. 
Despite sharing the same theoretical time complexity, these methods exhibit running time discrepancies, potentially due to differences in model expressiveness, which may require varying numbers of epochs to achieve convergence. 
In contrast, embedding-based methods are generally faster, consistent with our theoretical analysis. 
For the proposed models, \texttt{GlobAlign} shows outstanding efficiency compared to other OT-based methods, whereas \texttt{GlobAlign-E} achieves orders of magnitude speedup over them. 
Besides, its running time improvement against embedding-based models is also remarkable.

\begin{figure}[ht]
\centering
\begin{tabular}{cc} 
    \centering
    \hspace{0mm}\includegraphics[ height=35mm]{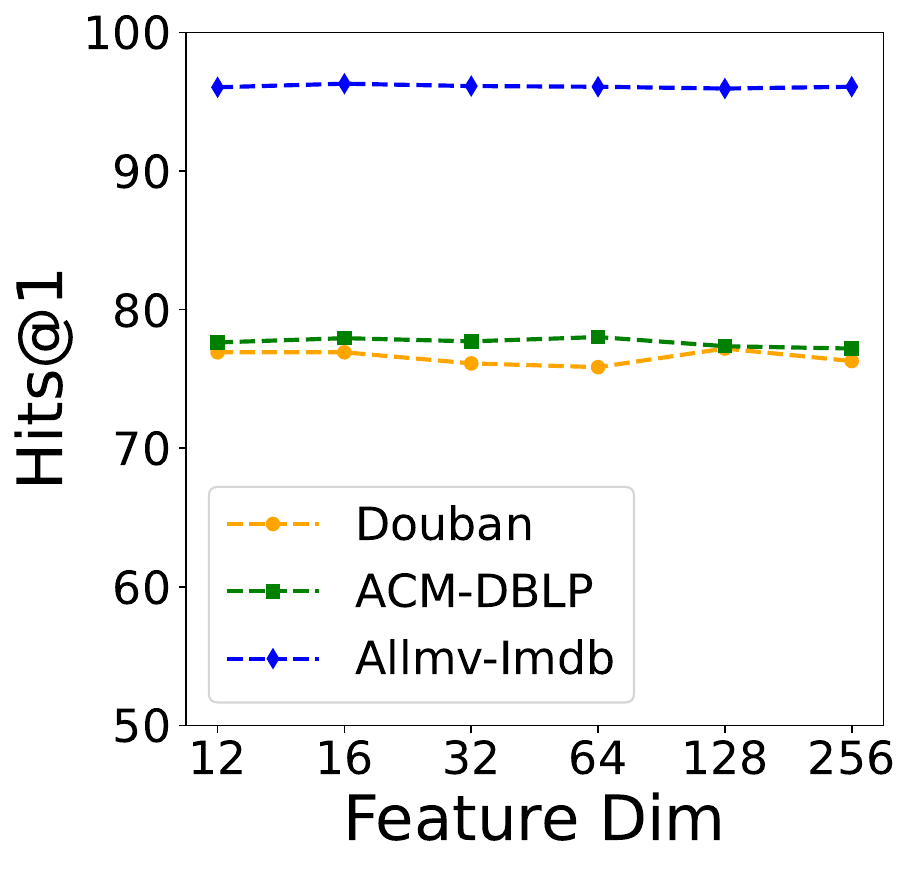} & 
    \hspace{0mm}\includegraphics[ height=35mm]{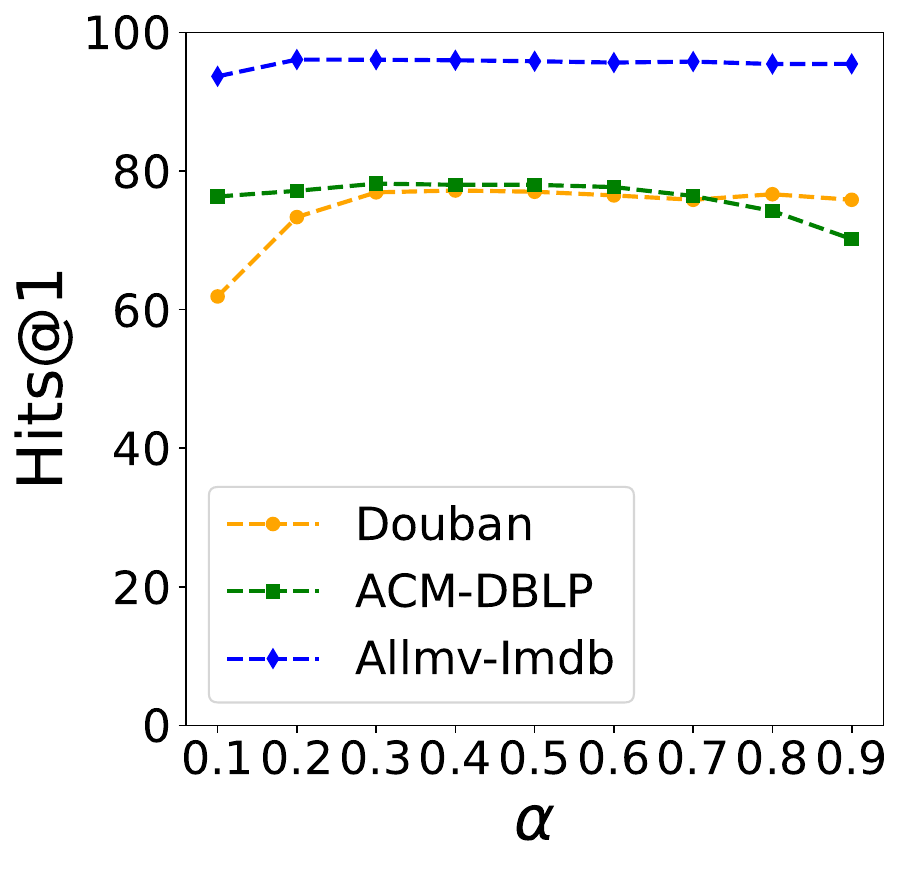} \\
\end{tabular}
\caption{Impact of hyperparameters}
\label{fig:hyp}
\end{figure}

\vspace{-3mm}

\begin{figure*}[!htbp]
\centering
\begin{tabular}{ccc} 
    \centering
    \hspace{0mm}\includegraphics[ height=30mm]{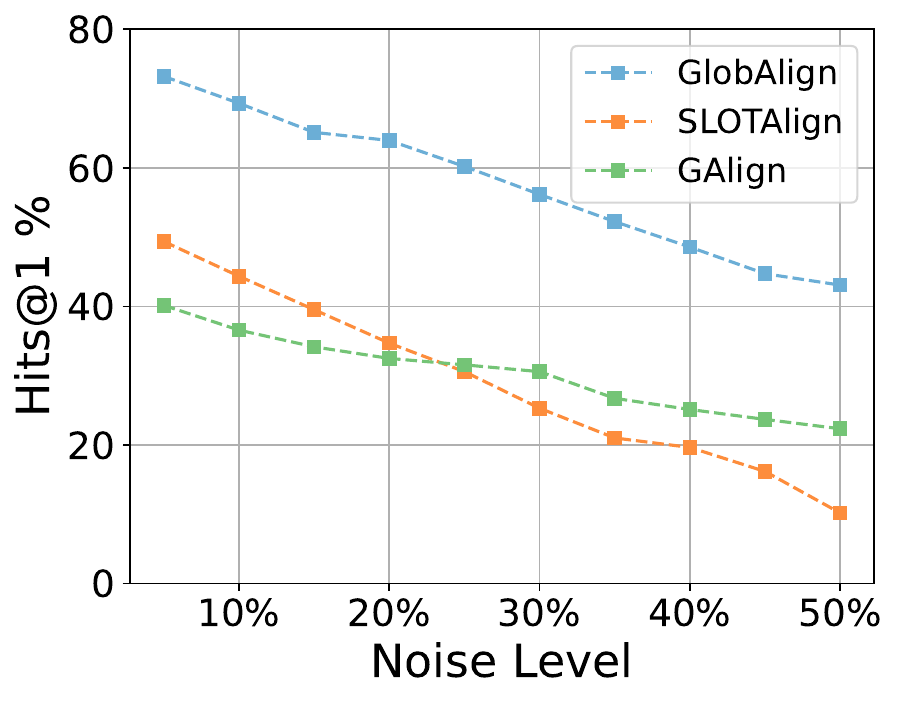} & 
    \hspace{0mm}\includegraphics[ height=30mm]{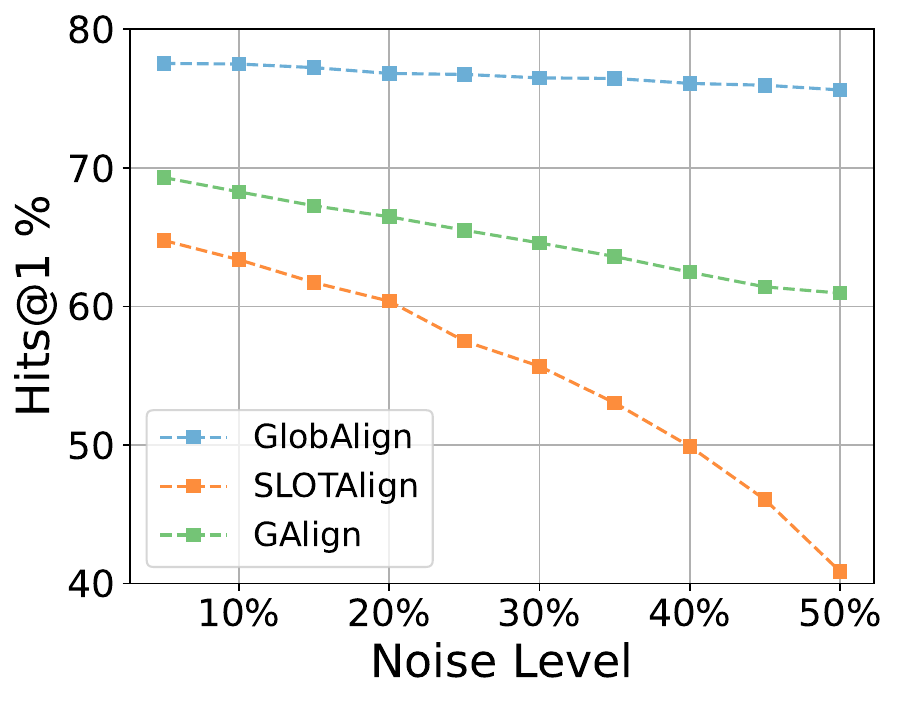} & 
    \hspace{0mm}\includegraphics[ height=30mm]{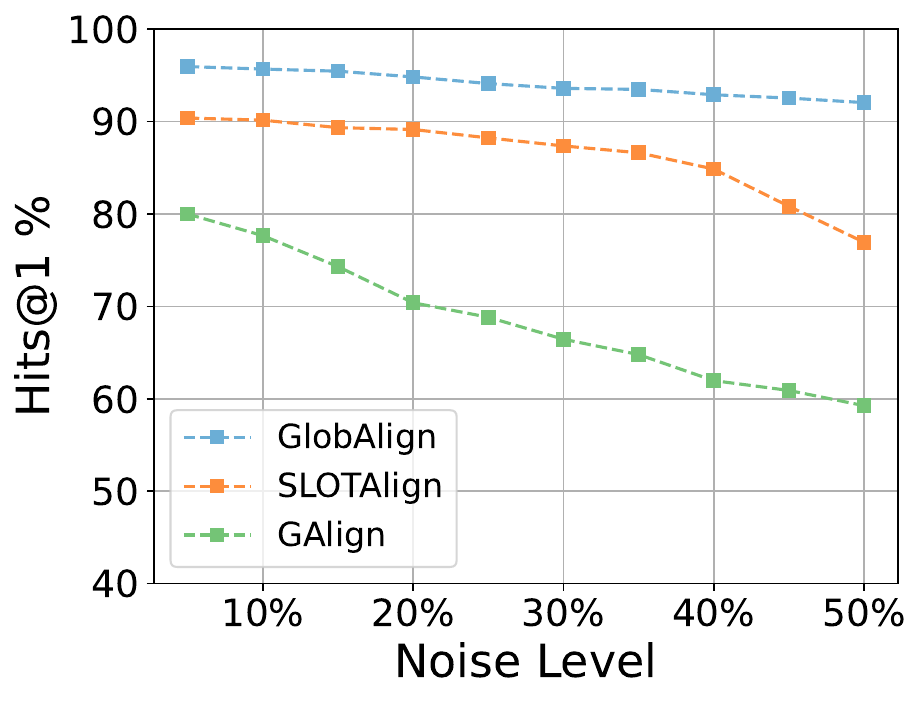}\\
    {Douban} & {ACM-DBLP} & {Allmv-Imdb} \\
\end{tabular}
\vspace{3mm}
\caption{Accuracy (Hits@1) vs. noise level}
\label{fig:robustness-hits1}
\end{figure*}

\subsection{Further Analysis of the Proposed Model}
\subsubsection{Robustness Analysis}
We evaluate the model's performance across three real datasets under various noise levels. 
We select two representative methods \texttt{SLOTAlign}~\cite{slotalign} and \texttt{GAlign}~\cite{galign} for comparison. 
Our model exhibits significantly greater robustness under different noise levels (Figure~\ref{fig:robustness-hits1}).
Notably, on the ACM-DBLP and Allmv-Imdb datasets, even when 50\% of the edges are perturbed, our model surpasses the performance of \texttt{GAlign} and \texttt{SLOTAlign} under noise-free conditions. 
We believe that methods solely based on local graph information struggle to capture relationships between nodes when a certain percentage of edges are absent.

\subsubsection{Sensitivity Analysis}
We also explore the model's performance by varying feature dimensions and with different weight coefficients for $\mathbf{Cost}_{gwd}$ and $\mathbf{Cost}_{wd}$. 
As shown in Figure~\ref{fig:hyp}, the model achieves stable performance across a wide range of feature dimensions. 
Additionally, as the weight parameter $\alpha$ in Equation~\ref{eqn:cross-cost} varies, the model's performance exhibits slight fluctuations, which is dataset-specific, reflecting the different contributions of $\mathbf{Cost}_{gwd}$ and $\mathbf{Cost}_{wd}$. In practice, we set $\alpha$ to 0.5, which strikes a balance between these two components.

\begin{figure}[ht]
\centering
\begin{tabular}{c} 
    \centering
    \hspace{-4mm}\includegraphics[ height=40mm]{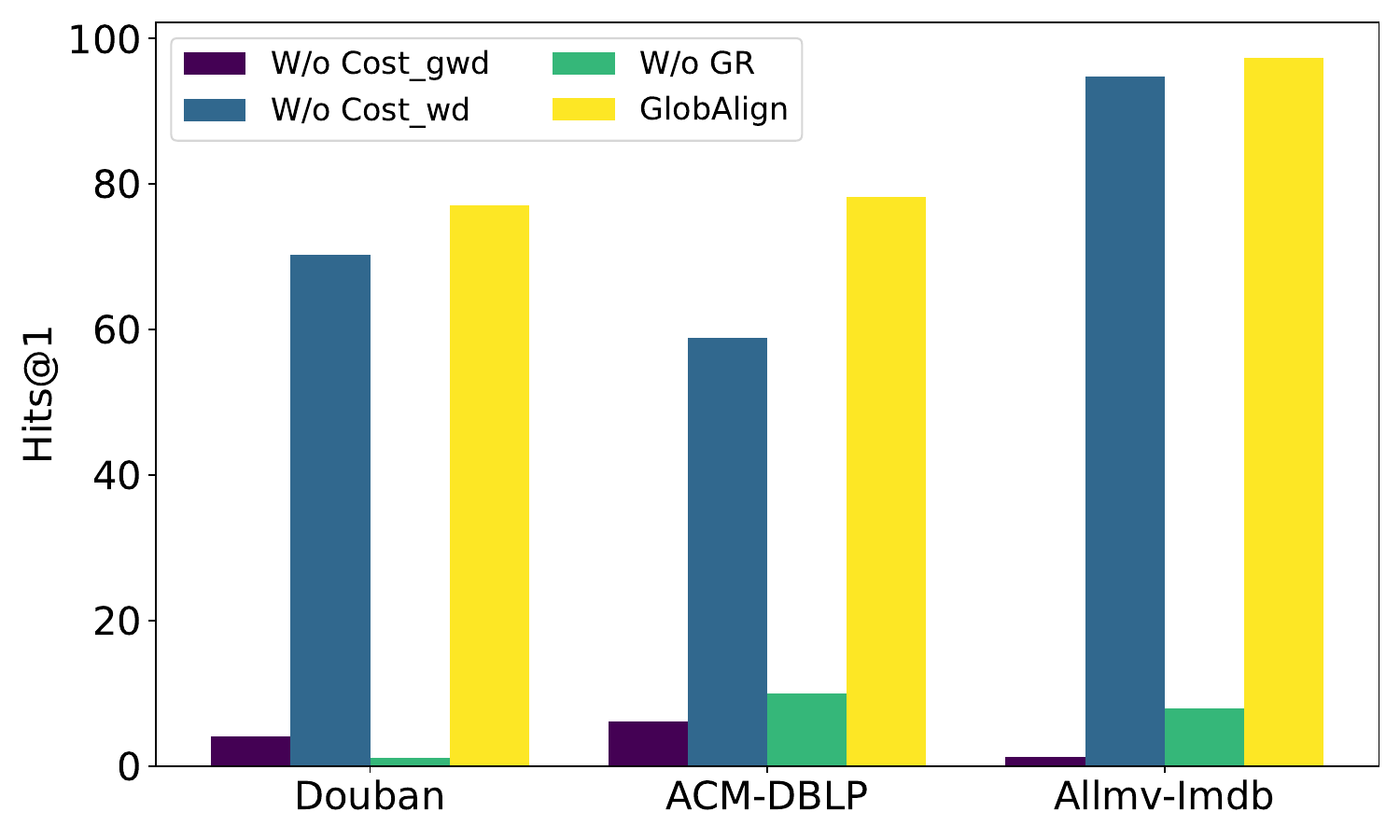} 
\end{tabular}
\vspace{-3mm}
\caption{Ablation study}
\label{fig:ablation}
\end{figure}

\vspace{-5mm}

\begin{figure}[ht]
\centering
\begin{tabular}{cc} 
    \centering
    \hspace{0mm}\includegraphics[height=35mm]{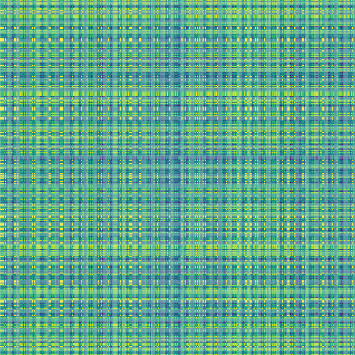} & 
    \hspace{0mm}\includegraphics[height=35mm]{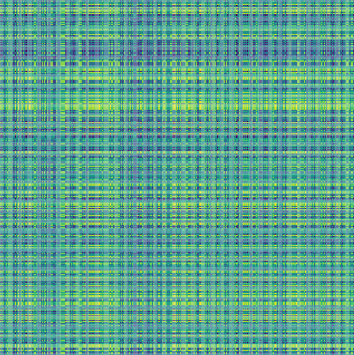} \\
    {(a) Offline Graph} &  {(b) Online Graph} \\
\end{tabular}
\vspace{3mm}
\caption{The attention matrices for the Douban dataset}
\label{fig:attention}
\end{figure}

\vspace{3mm}
\subsubsection{Ablation Study}
To investigate the impact of different components,
we conduct ablation studies on three real-world datasets with the following model variants.
W/o $\text{Cost}_{\text{gwd}}$ (resp. W/o $\text{Cost}_{\text{wd}}$) represents that \texttt{GlobAlign} omits the GWD (resp. WD) term in Equation~\ref{eqn:cross-cost}. 
W/o GR denotes that \texttt{GlobAlign} does not explicitly consider self-attention-based interactions in the GWD and WD terms, retaining only the local structural alignment in the GWD term. 

As shown in Figure~\ref{fig:ablation}, each component plays a crucial role, and the GWD-based cost is indispensable, without which \texttt{GlobAlign} cannot make a reasonable prediction. 
This can be attributed to the superior properties of GWD in modeling cross-domain problems. Moreover, when the hierarchical cross-graph transport cost loses global interaction information (the W/o GR variant), the performance degrades significantly, demonstrating the critical importance of our proposed global representation and alignment strategy. Finally, by comparing the model complexity and accuracy between \texttt{GlobAlign} and \texttt{GlobAlign-E}, we validate the effectiveness of the sparsification strategy.

\subsubsection{Analysis of Global Attention}
We visualize the attention matrix of Douban, as shown in Figure~\ref{fig:attention}. 
In general, the attention matrix is more dense and shares less resemblance to the graph topological structure. 
Thus, we believe that our global representation with the self-attention mechanism successfully extracts deeper influences between nodes for the graph alignment problem.

\section{Related Work}
Early works typically formulate the graph alignment problem as the maximum common subgraph isomorphism problem~\cite{ref3, ref23, ref40} or the quadratic assignment problem~\cite{ref15, ref22}, and return the exact alignment consisting a set of aligned node pairs. 
However, in practice, it is difficult to achieve exact graph alignment because of its NP-hardness~\cite{gwl}. 
Some classical methods also employ spectral functions~\cite{hermanns2023grasp} and linear assignment algorithms~\cite{kuhn1955hungarian,jonker1988shortest} to obtain matching results. These approaches typically have a computational complexity of at least $O(n^3)$.

In recent years, the research focus has shifted towards learning-based methods, which predict an alignment matrix indicating the alignment probability for every node pair across two graphs. 
These works can be roughly classified as \emph{embedding-based} and \emph{Optimal Transport (OT)-based} methods. 
Specifically, a line of research~\cite{walign,galign,gtcalign,peng2023robust,heimann2018regal} follows the ``embed-then-cross-compare'' paradigm.
First, node embeddings are generated through graph representation learning and then mapped into a unified feature space using transformation functions~\cite{walign} or weight-sharing mechanisms~\cite{galign,gtcalign}.
Then the alignment matrix is derived by computing the similarity between node embeddings, with some methods~\cite{peng2023robust,galign} further refining the results to improve prediction accuracy.

Another popular approach is to model the graph alignment problem using optimal transport (OT). 
These methods~\cite{gwl,slotalign,dhot,chen2025combalign,chen2025enhancing,chen2025leveraging} treat graphs as distributions and use OT and its extended version in metric spaces, the Gromov-Wasserstein distance~\cite{memoli2011gromov}, to address cross-domain challenges. 
Based on the intra-graph costs that incorporate node-wise interactions, the cross-graph node correspondences are established by minimizing the total transport cost between the two distributions. 

Meanwhile, there exists a plethora of works~\cite{jiao2024cina,final,nextalign,REGAL,yan2021bright,zeng2023parrot} tackling the supervised graph alignment problem, in which a certain proportion of node correspondences is known in advance. 
In particular, consistency-based approaches~\cite{final,nextalign} usually assume the existence of a noisy permutation between the aligned graphs and put emphasis on local topology and attribute consistency.
Embedding-based methods compute low-dimensional node embeddings through matrix factorization~\cite{REGAL} or Random Walk with Restart (RWR)~\cite{ yan2021bright}, ensuring anchor node pairs have close embeddings.
A recently proposed method~\cite{zeng2023parrot} formulates semi-supervised graph alignment as an optimal transport problem and devises transport cost inspired by RWR for the alignment process. 

We also find knowledge graph entity alignment a closely related research field with extensive literature~\cite{wang2020knowledge, sun2018bootstrapping,  zeng2021comprehensive, qi2021unsupervised,  liu2022selfkg, zeng2022interactive, mao2021alignment, mao2022lightea, tang2023fused}, aiming to identify matching entities while incorporating richer semantic information, an extra ingredient compared to our studied problem.

\vspace{-3mm}
\section{Conclusion}
We propose \texttt{GlobAlign}, a global representation and optimal transport-based approach for graph alignment without supervision. 
Motivated by the limitations of the ``local representation, global alignment'' paradigm adopted by existing methods, \texttt{GlobAlign} follows a newly proposed ``global representation and alignment'' paradigm based on our problem formalization and theoretical analysis. 
It achieves superior performance by incorporating the global attention mechanism and a hierarchical cross-graph transport cost, whereas its variant targeting for better efficiency, i.e., \texttt{GlobAlign-E}, bridges the time complexity gap between representative embedding and OT-based methods, demonstrating significant speedup over both of them. 
Our solution advances unsupervised graph alignment for both accuracy and efficiency, which is prominent compared to recent studies.

\section*{Declarations}

\begin{flushleft}
\textbf{Funding} This work was supported by National Natural Science Foundation of China under grant 62202037 (Yu Liu) and 62372031 (Youfang Lin). \\[10pt]
\textbf{Competing Interest} The authors declare that they have no competing interest.\\[10pt]
\textbf{Ethical Approval} Not applicable.\\[10pt]
\textbf{Availability of data and materials} 
The source code and data will be made publicly available after the peer review process of the paper. \\[10pt]
\textbf{Consent for publication}
We confirm that this manuscript has not been published elsewhere and is not under consideration by another journal. All authors have approved the manuscript and agree with its submission to World Wide Web.  \\[10pt]
\textbf{Author Contributions} Songyang Chen and Yu Liu proposed the method and completed the manuscript writing. Songyang Chen conducted all the experiments for the paper, while Youfang Lin, Shuai Zheng, and Lei Zou provided valuable revision suggestions. \\[10pt]
\textbf{Acknowledgement}
The authors gratefully acknowledge Prof. Dixin Luo and Prof. Hongteng Xu for their valuable contribution in providing the baseline code.
\end{flushleft}


\bibliography{sn-bibliography}

@String{Computing = "Computing" }

@String{Computer = "{IEEE} Computer" }

@String{Academic = "Academic Press" }

@String{Springer = "Springer-Verlag" }

@ArtifactSoftware{R,
    title = {R: A Language and Environment for Statistical Computing},
    author = {{R Core Team}},
    organization = {R Foundation for Statistical Computing},
    address = {Vienna, Austria},
    year = {2019},
    url = {https://www.R-project.org/},
}

@inproceedings{vaswani2017attention,
  title={Attention is all you need},
  author={Vaswani, Ashish and Shazeer, Noam and Parmar, Niki and Uszkoreit, Jakob and Jones, Llion and Gomez, Aidan N and Kaiser, {\L}ukasz and Polosukhin, Illia},
  booktitle={Advances in Neural Information Processing systems},
  pages={5998--6008},
  year={2017}
}

@inproceedings{slotalign,
  title={Robust attributed graph alignment via joint structure learning and optimal transport},
  author={Tang, Jianheng and Zhang, Weiqi and Li, Jiajin and Zhao, Kangfei and Tsung, Fugee and Li, Jia},
  booktitle={2023 IEEE 39th International Conference on Data Engineering (ICDE)},
  pages={1638--1651},
  year={2023},
  organization={IEEE}
}

@inproceedings{galign,
  title={Adaptive network alignment with unsupervised and multi-order convolutional networks},
  author={Trung, Huynh Thanh and Van Vinh, Tong and Tam, Nguyen Thanh and Yin, Hongzhi and Weidlich, Matthias and Hung, Nguyen Quoc Viet},
  booktitle={2020 IEEE 36th International Conference on Data Engineering (ICDE)},
  pages={85--96},
  year={2020},
  organization={IEEE}
}

@article{zhang2018attributed,
  title={Attributed network alignment: Problem definitions and fast solutions},
  author={Zhang, Si and Tong, Hanghang},
  journal={IEEE Transactions on Knowledge and Data Engineering},
  volume={31},
  number={9},
  pages={1680--1692},
  year={2018},
  publisher={IEEE}
}

@inproceedings{gwl,
  title={Gromov-wasserstein learning for graph matching and node embedding},
  author={Xu, Hongteng and Luo, Dixin and Zha, Hongyuan and Duke, Lawrence Carin},
  booktitle={International conference on machine learning},
  pages={6932--6941},
  year={2019},
  organization={PMLR}
}

@inproceedings{walign,
  title={Unsupervised graph alignment with wasserstein distance discriminator},
  author={Gao, Ji and Huang, Xiao and Li, Jundong},
  booktitle={Proceedings of the 27th ACM SIGKDD Conference on Knowledge Discovery \& Data Mining},
  pages={426--435},
  year={2021}
}

@inproceedings{zeng2023parrot,
  title={PARROT: Position-Aware Regularized Optimal Transport for Network Alignment},
  author={Zeng, Zhichen and Zhang, Si and Xia, Yinglong and Tong, Hanghang},
  booktitle={Proceedings of the ACM Web Conference 2023},
  pages={372--382},
  year={2023}
}

@article{gtcalign,
  title={GTCAlign: Global Topology Consistency-based Graph Alignment},
  author={Wang, Chenxu and Jiang, Peijing and Zhang, Xiangliang and Wang, Pinghui and Qin, Tao and Guan, Xiaohong},
  journal={IEEE Transactions on Knowledge and Data Engineering},
  year={2023},
  publisher={IEEE}
}

@article{gcn,
  title={Semi-supervised classification with graph convolutional networks},
  author={Kipf, Thomas N and Welling, Max},
  journal={arXiv preprint arXiv:1609.02907},
  year={2016}
}

@inproceedings{peyre2016gromov,
  title={Gromov-wasserstein averaging of kernel and distance matrices},
  author={Peyr{\'e}, Gabriel and Cuturi, Marco and Solomon, Justin},
  booktitle={International conference on machine learning},
  pages={2664--2672},
  year={2016},
  organization={PMLR}
}

@article{memoli2011gromov,
  title={Gromov--Wasserstein distances and the metric approach to object matching},
  author={M{\'e}moli, Facundo},
  journal={Foundations of computational mathematics},
  volume={11},
  pages={417--487},
  year={2011},
  publisher={Springer}
}

@article{kuhn1955hungarian,
  title={The Hungarian method for the assignment problem},
  author={Kuhn, Harold W},
  journal={Naval research logistics quarterly},
  volume={2},
  number={1-2},
  pages={83--97},
  year={1955},
  publisher={Wiley Online Library}
}

@inproceedings{yan2021bright,
  title={Bright: A bridging algorithm for network alignment},
  author={Yan, Yuchen and Zhang, Si and Tong, Hanghang},
  booktitle={Proceedings of the web conference 2021},
  pages={3907--3917},
  year={2021}
}

@article{dhot,
  title={DHOT-GM: Robust Graph Matching Using A Differentiable Hierarchical Optimal Transport Framework},
  author={Cheng, Haoran and Luo, Dixin and Xu, Hongteng},
  journal={arXiv preprint arXiv:2310.12081},
  year={2023}
}

@book{WD,
  title={Optimal transport: old and new},
  author={Villani, C{\'e}dric and others},
  volume={338},
  year={2009},
}

@article{sinkhorn,
  title={Sinkhorn distances: Lightspeed computation of optimal transport},
  author={Cuturi, Marco},
  journal={Advances in neural information processing systems},
  volume={26},
  year={2013}
}

@inproceedings{zhang2015multiple,
  title={Multiple anonymized social networks alignment},
  author={Zhang, Jiawei and Philip, S Yu},
  booktitle={2015 IEEE International Conference on Data Mining},
  pages={599--608},
  year={2015},
  organization={IEEE}
}

@inproceedings{tang2008arnetminer,
  title={Arnetminer: extraction and mining of academic social networks},
  author={Tang, Jie and Zhang, Jing and Yao, Limin and Li, Juanzi and Zhang, Li and Su, Zhong},
  booktitle={Proceedings of the 14th ACM SIGKDD international conference on Knowledge discovery and data mining},
  pages={990--998},
  year={2008}
}

@inproceedings{zhang2021balancing,
  title={Balancing consistency and disparity in network alignment},
  author={Zhang, Si and Tong, Hanghang and Jin, Long and Xia, Yinglong and Guo, Yunsong},
  booktitle={Proceedings of the 27th ACM SIGKDD conference on knowledge discovery \& data mining},
  pages={2212--2222},
  year={2021}
}

@inproceedings{li2019partially,
  title={Partially shared adversarial learning for semi-supervised multi-platform user identity linkage},
  author={Li, Chaozhuo and Wang, Senzhang and Wang, Hao and Liang, Yanbo and Yu, Philip S and Li, Zhoujun and Wang, Wei},
  booktitle={Proceedings of the 28th ACM international conference on information and knowledge management},
  pages={249--258},
  year={2019}
}

@inproceedings{heimann2018regal,
  title={Regal: Representation learning-based graph alignment},
  author={Heimann, Mark and Shen, Haoming and Safavi, Tara and Koutra, Danai},
  booktitle={Proceedings of the 27th ACM international conference on information and knowledge management},
  pages={117--126},
  year={2018}
}

@inproceedings{fusegwd,
  title={Optimal transport for structured data with application on graphs},
  author={Titouan, Vayer and Courty, Nicolas and Tavenard, Romain and Flamary, R{\'e}mi},
  booktitle={International Conference on Machine Learning},
  pages={6275--6284},
  year={2019},
  organization={PMLR}
}

@article{gin,
  title={How powerful are graph neural networks?},
  author={Xu, Keyulu and Hu, Weihua and Leskovec, Jure and Jegelka, Stefanie},
  journal={arXiv preprint arXiv:1810.00826},
  year={2018}
}

@article{tang2023fused,
  title={A fused gromov-wasserstein framework for unsupervised knowledge graph entity alignment},
  author={Tang, Jianheng and Zhao, Kangfei and Li, Jia},
  journal={arXiv preprint arXiv:2305.06574},
  year={2023}
}

@article{qi2021unsupervised,
  title={Unsupervised knowledge graph alignment by probabilistic reasoning and semantic embedding},
  author={Qi, Zhiyuan and Zhang, Ziheng and Chen, Jiaoyan and Chen, Xi and Xiang, Yuejia and Zhang, Ningyu and Zheng, Yefeng},
  journal={arXiv preprint arXiv:2105.05596},
  year={2021}
}

@inproceedings{liu2022selfkg,
  title={Selfkg: Self-supervised entity alignment in knowledge graphs},
  author={Liu, Xiao and Hong, Haoyun and Wang, Xinghao and Chen, Zeyi and Kharlamov, Evgeny and Dong, Yuxiao and Tang, Jie},
  booktitle={Proceedings of the ACM Web Conference 2022},
  pages={860--870},
  year={2022}
}

@inproceedings{zeng2022interactive,
  title={Interactive contrastive learning for self-supervised entity alignment},
  author={Zeng, Kaisheng and Dong, Zhenhao and Hou, Lei and Cao, Yixin and Hu, Minghao and Yu, Jifan and Lv, Xin and Cao, Lei and Wang, Xin and Liu, Haozhuang and others},
  booktitle={Proceedings of the 31st ACM International Conference on Information \& Knowledge Management},
  pages={2465--2475},
  year={2022}
}

@article{mao2021alignment,
  title={From alignment to assignment: Frustratingly simple unsupervised entity alignment},
  author={Mao, Xin and Wang, Wenting and Wu, Yuanbin and Lan, Man},
  journal={arXiv preprint arXiv:2109.02363},
  year={2021}
}

@article{mao2022lightea,
  title={Lightea: A scalable, robust, and interpretable entity alignment framework via three-view label propagation},
  author={Mao, Xin and Wang, Wenting and Wu, Yuanbin and Lan, Man},
  journal={arXiv preprint arXiv:2210.10436},
  year={2022}
}

@inproceedings{wang2018cross,
  title={Cross-lingual knowledge graph alignment via graph convolutional networks},
  author={Wang, Zhichun and Lv, Qingsong and Lan, Xiaohan and Zhang, Yu},
  booktitle={Proceedings of the 2018 conference on empirical methods in natural language processing},
  pages={349--357},
  year={2018}
}

@inproceedings{wang2020knowledge,
  title={Knowledge graph alignment with entity-pair embedding},
  author={Wang, Zhichun and Yang, Jinjian and Ye, Xiaoju},
  booktitle={Proceedings of the 2020 Conference on Empirical Methods in Natural Language Processing (EMNLP)},
  pages={1672--1680},
  year={2020}
}

@inproceedings{sun2018bootstrapping,
  title={Bootstrapping entity alignment with knowledge graph embedding.},
  author={Sun, Zequn and Hu, Wei and Zhang, Qingheng and Qu, Yuzhong},
  booktitle={IJCAI},
  volume={18},
  number={2018},
  year={2018}
}

@article{zeng2021comprehensive,
  title={A comprehensive survey of entity alignment for knowledge graphs},
  author={Zeng, Kaisheng and Li, Chengjiang and Hou, Lei and Li, Juanzi and Feng, Ling},
  journal={AI Open},
  volume={2},
  pages={1--13},
  year={2021},
  publisher={Elsevier}
}

@article{wu2024simplifying,
  title={Simplifying and empowering transformers for large-graph representations},
  author={Wu, Qitian and Zhao, Wentao and Yang, Chenxiao and Zhang, Hengrui and Nie, Fan and Jiang, Haitian and Bian, Yatao and Yan, Junchi},
  journal={Advances in Neural Information Processing Systems},
  volume={36},
  year={2024}
}

@article{wu2023difformer,
  title={Difformer: Scalable (graph) transformers induced by energy constrained diffusion},
  author={Wu, Qitian and Yang, Chenxiao and Zhao, Wentao and He, Yixuan and Wipf, David and Yan, Junchi},
  journal={arXiv preprint arXiv:2301.09474},
  year={2023}
}

@inproceedings{liu2017novel,
  title={Novel geometric approach for global alignment of PPI networks},
  author={Liu, Yangwei and Ding, Hu and Chen, Danyang and Xu, Jinhui},
  booktitle={Proceedings of the AAAI Conference on Artificial Intelligence},
  volume={31},
  number={1},
  year={2017}
}

@article{bolte2014proximal,
  title={Proximal alternating linearized minimization for nonconvex and nonsmooth problems},
  author={Bolte, J{\'e}r{\^o}me and Sabach, Shoham and Teboulle, Marc},
  journal={Mathematical Programming},
  volume={146},
  number={1},
  pages={459--494},
  year={2014},
  publisher={Springer}
}

@article{li2023efficient,
  title={Efficient approximation of Gromov-Wasserstein distance using importance sparsification},
  author={Li, Mengyu and Yu, Jun and Xu, Hongteng and Meng, Cheng},
  journal={Journal of Computational and Graphical Statistics},
  volume={32},
  number={4},
  pages={1512--1523},
  year={2023},
  publisher={Taylor \& Francis}
}

@inproceedings{final,
  title={Final: Fast attributed network alignment},
  author={Zhang, Si and Tong, Hanghang},
  booktitle={Proceedings of the 22nd ACM SIGKDD international conference on knowledge discovery and data mining},
  pages={1345--1354},
  year={2016}
}

@inproceedings{peng2023robust,
  title={Robust network alignment with the combination of structure and attribute embeddings},
  author={Peng, Jingkai and Xiong, Fei and Pan, Shirui and Wang, Liang and Xiong, Xi},
  booktitle={2023 IEEE International Conference on Data Mining (ICDM)},
  pages={498--507},
  year={2023},
  organization={IEEE}
}

@inproceedings{nextalign,
  title={Balancing consistency and disparity in network alignment},
  author={Zhang, Si and Tong, Hanghang and Jin, Long and Xia, Yinglong and Guo, Yunsong},
  booktitle={Proceedings of the 27th ACM SIGKDD conference on knowledge discovery \& data mining},
  pages={2212--2222},
  year={2021}
}

@inproceedings{REGAL,
  title={Regal: Representation learning-based graph alignment},
  author={Heimann, Mark and Shen, Haoming and Safavi, Tara and Koutra, Danai},
  booktitle={Proceedings of the 27th ACM international conference on information and knowledge management},
  pages={117--126},
  year={2018}
}

@inproceedings{jiao2024cina,
  title={CINA: Curvature-Based Integrated Network Alignment with Hypergraph},
  author={Jiao, Pengfei and Liu, Yuanqi and Wang, Yinghui and Zhang, Ge},
  booktitle={2024 IEEE 40th International Conference on Data Engineering (ICDE)},
  pages={2709--2722},
  year={2024},
  organization={IEEE}
}

@article{cuturi2013sinkhorn,
  title={Sinkhorn distances: Lightspeed computation of optimal transport},
  author={Cuturi, Marco},
  journal={Advances in neural information processing systems},
  volume={26},
  year={2013}
}

@article{sinkhorn1967concerning,
  title={Concerning nonnegative matrices and doubly stochastic matrices},
  author={Sinkhorn, Richard and Knopp, Paul},
  journal={Pacific Journal of Mathematics},
  volume={21},
  number={2},
  pages={343--348},
  year={1967},
  publisher={Mathematical Sciences Publishers}
}

@inproceedings{kantorovich1942transfer,
  title={On the transfer of masses (in Russian)},
  author={Kantorovich, L},
  booktitle={Doklady Akademii Nauk},
  volume={37},
  pages={227},
  year={1942}
}

@article{shchur2018pitfalls,
  title={Pitfalls of graph neural network evaluation},
  author={Shchur, Oleksandr and Mumme, Maximilian and Bojchevski, Aleksandar and G{\"u}nnemann, Stephan},
  journal={arXiv preprint arXiv:1811.05868},
  year={2018}
}

@article{hermanns2023grasp,
  title={Grasp: Scalable graph alignment by spectral corresponding functions},
  author={Hermanns, Judith and Skitsas, Konstantinos and Tsitsulin, Anton and Munkhoeva, Marina and Kyster, Alexander and Nielsen, Simon and Bronstein, Alexander M and Mottin, Davide and Karras, Panagiotis},
  journal={ACM Transactions on Knowledge Discovery from Data},
  volume={17},
  number={4},
  pages={1--26},
  year={2023},
  publisher={ACM New York, NY}
}

@inproceedings{jonker1988shortest,
  title={A shortest augmenting path algorithm for dense and sparse linear assignment problems},
  author={Jonker, Roy and Volgenant, Ton},
  booktitle={DGOR/NSOR: Papers of the 16th Annual Meeting of DGOR in Cooperation with NSOR/Vortr{\"a}ge der 16. Jahrestagung der DGOR zusammen mit der NSOR},
  pages={622--622},
  year={1988},
  organization={Springer}
}

@article{xu2019scalable,
  title={Scalable Gromov-Wasserstein learning for graph partitioning and matching},
  author={Xu, Hongteng and Luo, Dixin and Carin, Lawrence},
  journal={Advances in neural information processing systems},
  volume={32},
  year={2019}
}

@inproceedings{over-smoothing,
  title={Deeper insights into graph convolutional networks for semi-supervised learning},
  author={Li, Qimai and Han, Zhichao and Wu, Xiao-Ming},
  booktitle={Proceedings of the AAAI conference on artificial intelligence},
  volume={32},
  number={1},
  year={2018}
}

@article{over-squashing,
  title={On the bottleneck of graph neural networks and its practical implications},
  author={Alon, Uri and Yahav, Eran},
  journal={arXiv preprint arXiv:2006.05205},
  year={2020}
}

@inproceedings{ref3,
  author       = {Mohsen Bayati and
                  Margot Gerritsen and
                  David F. Gleich and
                  Amin Saberi and
                  Ying Wang},
  editor       = {Wei Wang and
                  Hillol Kargupta and
                  Sanjay Ranka and
                  Philip S. Yu and
                  Xindong Wu},
  title        = {Algorithms for Large, Sparse Network Alignment Problems},
  booktitle    = {{ICDM} 2009, The Ninth {IEEE} International Conference on Data Mining,
                  Miami, Florida, USA, 6-9 December 2009},
  pages        = {705--710},
  publisher    = {{IEEE} Computer Society},
  year         = {2009},
  url          = {https://doi.org/10.1109/ICDM.2009.135},
  doi          = {10.1109/ICDM.2009.135},
  bibsource    = {dblp computer science bibliography, https://dblp.org}
}

@article{ref23,
  author       = {Gunnar W. Klau},
  title        = {A new graph-based method for pairwise global network alignment},
  journal      = {{BMC} Bioinform.},
  volume       = {10},
  number       = {{S-1}},
  year         = {2009},
  url          = {https://doi.org/10.1186/1471-2105-10-S1-S59},
  doi          = {10.1186/1471-2105-10-S1-S59},
  bibsource    = {dblp computer science bibliography, https://dblp.org}
}

@inproceedings{ref40,
  author       = {Rohit Singh and
                  Jinbo Xu and
                  Bonnie Berger},
  editor       = {Terence P. Speed and
                  Haiyan Huang},
  title        = {Pairwise Global Alignment of Protein Interaction Networks by Matching
                  Neighborhood Topology},
  booktitle    = {Research in Computational Molecular Biology, 11th Annual International
                  Conference, {RECOMB} 2007, Oakland, CA, USA, April 21-25, 2007, Proceedings},
  series       = {Lecture Notes in Computer Science},
  volume       = {4453},
  pages        = {16--31},
  publisher    = {Springer},
  year         = {2007},
  url          = {https://doi.org/10.1007/978-3-540-71681-5\_2},
  doi          = {10.1007/978-3-540-71681-5\_2},
  bibsource    = {dblp computer science bibliography, https://dblp.org}
}

@article{ref15,
  author       = {Soheil Feizi and
                  Gerald T. Quon and
                  Mariana Recamonde Mendoza and
                  Muriel M{\'{e}}dard and
                  Manolis Kellis and
                  Ali Jadbabaie},
  title        = {Spectral Alignment of Graphs},
  journal      = {{IEEE} Trans. Netw. Sci. Eng.},
  volume       = {7},
  number       = {3},
  pages        = {1182--1197},
  year         = {2020},
  url          = {https://doi.org/10.1109/TNSE.2019.2913233},
  doi          = {10.1109/TNSE.2019.2913233},
  bibsource    = {dblp computer science bibliography, https://dblp.org}
}

@inproceedings{ref22,
  author       = {Paris A. Karakasis and
                  Aritra Konar and
                  Nicholas D. Sidiropoulos},
  editor       = {Feida Zhu and
                  Beng Chin Ooi and
                  Chunyan Miao},
  title        = {Joint Graph Embedding and Alignment with Spectral Pivot},
  booktitle    = {{KDD} '21: The 27th {ACM} {SIGKDD} Conference on Knowledge Discovery
                  and Data Mining, Virtual Event, Singapore, August 14-18, 2021},
  pages        = {851--859},
  publisher    = {{ACM}},
  year         = {2021},
  url          = {https://doi.org/10.1145/3447548.3467377},
  doi          = {10.1145/3447548.3467377},
  bibsource    = {dblp computer science bibliography, https://dblp.org}
}

@article{pagerank,
  title={The anatomy of a large-scale hypertextual web search engine},
  author={Brin, Sergey and Page, Lawrence},
  journal={Computer networks and ISDN systems},
  volume={30},
  number={1-7},
  pages={107--117},
  year={1998},
  publisher={Elsevier}
}

@techreport{page1999pagerank,
  title={The PageRank citation ranking: Bringing order to the web.},
  author={Page, Lawrence and Brin, Sergey and Motwani, Rajeev and Winograd, Terry},
  year={1999},
  institution={Stanford infolab}
}

@inproceedings{PPRGo,
  author       = {Aleksandar Bojchevski and
                  Johannes Klicpera and
                  Bryan Perozzi and
                  Amol Kapoor and
                  Martin Blais and
                  Benedek R{\'{o}}zemberczki and
                  Michal Lukasik and
                  Stephan G{\"{u}}nnemann},
  editor       = {Rajesh Gupta and
                  Yan Liu and
                  Jiliang Tang and
                  B. Aditya Prakash},
  title        = {Scaling Graph Neural Networks with Approximate PageRank},
  booktitle    = {{KDD} '20: The 26th {ACM} {SIGKDD} Conference on Knowledge Discovery
                  and Data Mining, Virtual Event, CA, USA, August 23-27, 2020},
  pages        = {2464--2473},
  publisher    = {{ACM}},
  year         = {2020},
  url          = {https://doi.org/10.1145/3394486.3403296},
  doi          = {10.1145/3394486.3403296},
  bibsource    = {dblp computer science bibliography, https://dblp.org}
}

@article{chen2025enhancing,
  title={Enhancing robust semi-supervised graph alignment via adaptive optimal transport},
  author={Chen, Songyang and Lin, Youfang and Liu, Yu and Ouyang, Yuwei and Guo, Zongshen and Zou, Lei},
  journal={World Wide Web},
  volume={28},
  number={2},
  pages={22},
  year={2025},
  publisher={Springer}
}

@article{chen2025combalign,
  title={CombAlign: Enhancing Model Expressiveness in Unsupervised Graph Alignment},
  author={Chen, Songyang and Liu, Yu and Zou, Lei and Wang, Zexuan and Lin, Youfang},
  journal={IEEE Transactions on Knowledge and Data Engineering},
  volume={38},
  number={2},
  pages={956--968},
  year={2025},
  publisher={IEEE}
}

@article{chen2025leveraging,
  title={Leveraging Attribute Interaction and Self-Training for Graph Alignment via Optimal Transport},
  author={Chen, Songyang and Lin, Youfang and Zeng, Ziyuan and Xue, Mengyang},
  journal={Mathematics},
  volume={13},
  number={12},
  pages={1971},
  year={2025},
  publisher={MDPI}
}

\end{document}